%% file: nonparametric_mixture.tex
\documentclass{article}

\usepackage{graphicx} 
\usepackage{subfigure}

\usepackage{natbib}

\usepackage{algorithm}
\usepackage{algorithmic}

\usepackage{hyperref}

\usepackage{fullpage}
\usepackage{authblk}

\usepackage{Definitions}
\newcommand{\hmu}{\widehat{\mu}}

\newcommand\tl{\widetilde}

\input{tensor-macros}
\usepackage{tikz}
\usetikzlibrary{calc}
\usepackage{exscale,relsize}
\usepackage{enumitem}

\usepackage{color}

\title{Nonparametric Estimation of Multi-View Latent Variable Models}

\author[1]{Le Song \thanks{Email: lsong@cc.gatech.edu}}
\author[2]{Animashree Anandkumar \thanks{Email: a.anandkumar@uci.edu}}
\author[1]{Bo Dai \thanks{Email: bodai@gatech.edu}}
\author[1]{Bo Xie \thanks{Email: zixu1986@gmail.com}}
\affil[1]{College of Computing, Georgia Institute of Technology}
\affil[2]{EECS, University of California Irvine}

\begin{document}
\maketitle

\begin{abstract}
		Spectral methods have greatly advanced the estimation of latent variable models, generating a sequence of novel and efficient algorithms with strong theoretical guarantees. However, current spectral algorithms are largely restricted to mixtures of discrete or Gaussian distributions. In this paper, we propose a kernel method for learning multi-view latent variable models, allowing each mixture component to be nonparametric.
		The key idea of the method is to embed the joint distribution of a multi-view latent variable into a reproducing kernel Hilbert space, and then the latent parameters are recovered using a robust tensor power method. We establish that the  sample complexity for the proposed method is quadratic in the number of latent components and is a low order polynomial in the other relevant parameters. Thus, our non-parametric tensor approach to learning latent variable models enjoys good sample and computational efficiencies. Moreover, the non-parametric tensor power method compares favorably to EM algorithm and other existing spectral algorithms in our experiments.
\end{abstract}

\section{Introduction} \label{sec:intro}

\setlength{\abovedisplayskip}{4pt}
\setlength{\abovedisplayshortskip}{1pt}
\setlength{\belowdisplayskip}{4pt}
\setlength{\belowdisplayshortskip}{1pt}
\setlength{\jot}{3pt}
\setlength{\textfloatsep}{3ex}

Latent variable models have been used to address various machine learning
problems, ranging from modeling temporal dynamics, to text document analysis and to social network analysis~\citep{RabJua86,Clark90,HofRafHan02,BleNgJor03}.
Recently, there is a surge of interest in designing spectral algorithms for estimating the parameters of latent variable models~\citep{HsuKakZha09,ParSonXin11,SonParXin11,FosRodUng12,AnandkumarEtal:tensor12, AnandkumarEtal:twosvd12, Franz13}. Compared to the Expectation-Maximization (EM) algorithm \citep{DemLaiRub77} traditionally used for this task, spectral algorithms are better in terms of their computational efficiency and provable guarantees.
Current spectral algorithms are largely restricted to mixture
of discrete or Gaussian distributions, e.g.~\citep{AnandkumarEtal:tensor12,Hsu13}. When the mixture components are distributions other than these standard distributions, the theoretical guarantees for these algorithms are no longer applicable, and their empirical performance can be very poor.

In this paper, we propose a kernel method for estimating the parameters of multi-view latent variable models where the mixture components can be nonparametric. The key idea  is to embed the joint distribution of such a model into a reproducing kernel Hilbert space, and exploit the low rank structure of the embedded distribution (or covariance operators). The key computation   involves a kernel singular value decomposition of the two-view covariance operator, followed by a robust tensor power method on the three-view covariance operator. These standard matrix operations makes the algorithm very efficient and easy to deploy.

The kernel algorithm proposed in this paper is more general than the previous spectral algorithms which work only for distributions with parametric assumptions~\citep{AnandkumarEtal:tensor12,Hsu13}. When we use delta kernel, our algorithm reduces to the spectral algorithm for discrete mixture components analyzed in~\citep{AnandkumarEtal:tensor12}. When we use universal kernels, such as Gaussian RBF kernel, our algorithm can recover Gaussian mixture components as well as mixture components with other distributions. In this sense, our work also provides a unifying framework for previous spectral algorithms. We prove sample complexity bounds for the nonparametric tensor power method and establish  that the sample complexity is quadratic in the number of latent components, and is a low order polynomial in the other relevant parameters such as the lower bound on mixing weights. Thus, we propose a computational and sample efficient nonparametric approach to learning latent variable models.

 Kernel methods have  been previously applied to learning latent variable models. However, none of the previous works explicitly recovers the actual parameters of the models~\cite{SonParXin11, SonDai13, SgoJanPetSch13}. Most of them estimate an (unknown) invertible  transformation of the latent parameters, and it is not clear how one can recover the actual parameters based on these estimates. Furthermore, these works focused on predictive task: recover the marginal distribution of the observed variables by making use of the low rank structure of the latent variable models. It is significantly more challenging to design kernel algorithms for actual parameter recovery and analyze theoretical properties of these algorithms.

We compare our kernel algorithm to the EM algorithm and previous spectral algorithms. We show that when the model assumptions are correct for the EM algorithm and previous spectral algorithms, our algorithm converges in terms of estimation error to these competitors. In the opposite cases when the model assumptions are incorrect, our algorithm is able to adapt to the nonparametric mixture components and beating alternatives by a very large margin.

\section{Notation}

We  denote by $X$ a random variable with domain $\Xcal$,
and refer to instantiations of $X$ by the lower case character, $x$.
We endow $\Xcal$ with some $\sigma$-algebra $\Ascr$ and denote a distributions (with respect to $\Ascr$) on $\Xcal$ by $\PP(X)$. We  also deal with multiple random variables, $X_1, X_2, \ldots, X_{\ell}$, with joint distribution $\PP(X_1,X_2,\ldots,X_{\ell})$. For simplicity of notation, we assume that the domains of all $X_t, t \in [\ell]$ are the same, but the methodology applies to the cases where they have different domains. Furthermore, we denote by $H$ a hidden variable with domain $\Hcal$ and distribution $\PP(H)$.

A \emph{reproducing kernel Hilbert space (RKHS)} $\Fcal$ on $\Xcal$ with a kernel $k(x,x')$ is a Hilbert space of
functions $f(\cdot):\Xcal \mapsto \RR$ with inner product $\inner{\cdot}{\cdot}_{\Fcal}$. Its element $k(x,\cdot)$ satisfies the reproducing property:
$\inner{f(\cdot)}{k(x, \cdot)}_{\Fcal} = f(x)$, and consequently, $\inner{k(x,\cdot)}{k(x', \cdot)}_{\Fcal} = k(x,x')$,
meaning that we can view the evaluation of a function $f$ at any point $x\in\Xcal$ as an inner product. Alternatively, $k(x,\cdot)$ can  be viewed as an implicit feature map $\phi(x)$ where $k(x,x')=\inner{\phi(x)}{\phi(x')}_{\Fcal}$.
Popular kernel functions on $\RR^n$ include the Gaussian RBF kernel $k(x,x') = \exp(-s
  \nbr{x-x'}^2)$ and the Laplace kernel $\exp(-s \|x - x'\|)$. Kernel functions have also been defined on
graphs, time series, dynamical systems, images and other structured
objects \, \cite{SchTsuVer04}. Thus the methodology presented below can be readily generalized to a diverse range of data types as long as kernel functions are defined.

\section{Kernel Embedding of Distributions}
\label{sec:embedding}

We begin by providing an overview of kernel embeddings of distributions, which are \emph{implicit} mappings of distributions into potentially \emph{infinite} dimensional RKHS.
The kernel embedding approach represents a distribution by an element in the RKHS associated with a kernel function \, \cite{SmoGreSonSch07,SriGreFukLanetal08},
\begin{align}
  \mu_{X} \, := \, \EE_{X} \sbr{\phi(X)} \, = \, \int_{\Xcal} \phi(x) \, \PP(dx),  \label{eq:embedding}
\end{align}
where the distribution is mapped to its expected feature map,~\ie,~to a point in a potentially infinite-dimensional and implicit feature space.
 The kernel embedding $\mu_{X}$ has the property that the expectation of any RKHS function $f$ can be evaluated as an inner product in $\Fcal$,
$
  \EE_{X} [f(X)] = \inner{\mu_{X}}{f}_{\Fcal}, \, \forall f\in\Fcal.
$

Kernel embeddings can be readily generalized to joint distributions of two or more variables using tensor product feature maps. For instance, we can embed a joint distribution of two variables $X_1$ and $X_2$ into a tensor product feature space $\Fcal\times \Fcal$ by
\begin{align}
    \Ccal_{X_1X_2} \, &:= \, \EE_{X_1X_2}[\phi(X_1)\otimes \phi(X_2)] \\
    \, &= \, \int_{\Xcal \times \Xcal} \phi(x_1) \otimes \phi(x_2) \, \PP(dx_1 \times dx_2),
\end{align}
where the reproducing kernel for the tensor product features satisfies
$
	\inner{\phi(x_1)\otimes \phi(x_2) }{\phi(x_1')\otimes \phi(x_2') }_{\Fcal\times \Fcal} \, = \,  k(x_1,x_1')\,k(x_2,x_2').
$
By analogy, we can also define $\Ccal_{X_1X_2X_3} := \EE_{X_1X_2X_3}[\phi(X_1)\otimes\phi(X_2)\otimes \phi(X_3)]$.

Kernel embedding of distributions has rich representational power. The mapping is injective for characteristic kernels~\cite{SriGreFukLanetal08}. That is, if two distributions, $\PP(X)$ and $\QQ(X)$, are different, they are mapped to two distinct points in the RKHS. For domain $\RR^d$, many commonly used kernels are characteristic, such as the Gaussian RBF kernel and Laplace kernel.
This injective property of kernel embeddings has been exploited to design  state-of-the-art two-sample tests~\cite{GreBorRasSchetal12} and  independence tests~\cite{GreFukTeoSonetal08}.

\subsection{Kernel Embedding as Multi-Linear Operator}

The joint embeddings can also be viewed as an uncentered covariance operator $\Ccal_{X_1X_2}:\Fcal\mapsto \Fcal$ by the standard equivalence between a tensor product feature and a linear map.
That is, given two functions $f_1,f_2\in\Fcal$, their covariance can be computed by
$
    \EE_{X_1X_2}[f_1(X_1)f_2(X_2)]=\inner{f_1}{\Ccal_{X_1X_2} f_2}_{\Fcal}
$
, or equivalently
$
\inner{f_1\otimes f_2}{\Ccal_{X_1X_2}}_{\Fcal\times\Fcal},
$
where in the former we view $\Ccal_{XY}$ as an operator while in the latter we view it as an element in tensor product feature space.
By analogy, $\Ccal_{X_1X_2X_3}$ can be regarded as a multi-linear operator from $\Fcal\times\Fcal\times\Fcal$ to $\RR$.
It will be clear from the context whether we use $\Ccal_{XY}$ as an operator between two spaces or as an element from a tensor product feature space. For generic introduction to tensor and tensor notation, please see~\cite{KolBad09}.

The operator $\Ccal_{X_1X_2X_3}$ (with shorthand $\Ccal_{X_{1:3}}$) is linear in each argument (mode) when fixing other arguments. Furthermore, the application of the operator to a set of elements $\cbr{f_1,f_2,f_3 \in \Fcal}$ can be defined using the inner product from the tensor product feature space,~\ie,
\begin{align}
	\Ccal_{X_{1:3}} \times_1 f_1 \times_2 \times_3 f_3
	&:= 	\inner{\Ccal_{X_{1:3}}}{\;f_1\otimes f_2\otimes f_3}_{\Fcal^3} \nonumber \\
	&=\EE_{X_1X_2X_3}\sbr{\prod_{i \in [3]} \inner{\phi(X_i)}{~f_i}_{\Fcal}}, \nonumber
\end{align}
where $\times_i$ means applying $f_i$ to the $i$-th argument of $\Ccal_{X_{1:3}}$. Furthermore, we can define the Hilbert-Schmidt norm $\nbr{\cdot}_{}$ of $\Ccal_{X_{1:3}}$ as
\[
 \nbr{\Ccal_{X_{1:3}}}_{}^2 = \sum_{i_1 = 1}^\infty \sum_{i_2 = 1}^\infty \sum_{i_3 = 1}^\infty \rbr{\Ccal_{X_{1:3}} \times_1 u_{i_1} \times_2 u_{i_2} \times_3 u_{i_3}}^2
\]
using three collections of orthonormal bases $\cbr{u_{i_1}}_{i_1=1}^\infty$, $\cbr{u_{i_2}}_{i_2=1}^\infty$, and $\cbr{u_{i_3}}_{i_3=1}^\infty$. We can also define the inner product for the space of such operator that $\nbr{\Ccal_{X_{1:3}}}_{}<\infty$
\begin{align}
	\inner{\Ccal_{X_{1:3}}}{~\widetilde{\Ccal}_{X_{1:3}}}_{}  &=  \sum_{i_1 = 1}^\infty \sum_{i_2 = 1}^\infty \sum_{i_\ell = 1}^\infty \rbr{\Ccal_{X_{1:\ell}} \times_1 u_{i_1} \times_2 u_{i_2} \times_3 u_{i_3}} \nonumber \\
	& \cdot\ (\widetilde{\Ccal}_{X_{1:\ell}} \times_1 u_{i_1} \times_2 \ldots \times_\ell u_{i_\ell}). \nonumber
\end{align}
The joint embedding, $\Ccal_{X_1 X_2}$, is a 2nd order tensor, and we can essentially use notations and operations for matrices. For instance, we can perform singular value decomposition
\[
    \Ccal_{X_1X_2} = \sum_{i=1}^{\infty} \sigma_i \cdot u_{i_1} \otimes u_{i_2},
\]
where $\sigma_i \in \RR$ are singular values ordered in nonincreasing manner, and $\cbr{u_{i_1}}_{i_1=1}^{\infty} \subset \Fcal, \cbr{u_{i_2}}_{i_2=1}^{\infty} \subset \Fcal$ are singular vectors and orthonormal bases. The rank of $\Ccal_{X_1X_2}$ is the smallest $k$ such that $\sigma_i = 0$ for $i > k$.

\subsection{Finite Sample Estimate}

While we rarely have access to the true underlying distribution, $\PP(X)$,
we can readily estimate its embedding using a finite sample average. Given a sample $\Dcal_{X} = \cbr{x^1, \ldots, x^m}$ of size $m$ drawn~\iid~from $\PP(X)$, the empirical kernel embedding is
\begin{align}
    \hmu_{X} &:= \frac{1}{m} \sum\nolimits_{i=1}^m \phi(x^i). \label{eq:empirical_embedding}
\end{align}
This empirical estimate converges to its population counterpart in RKHS norm, $\|\hmu_X - \mu_X \|_{\Fcal}$, with a rate of $O_p(m^{-\frac{1}{2}})$~\cite{SmoGreSonSch07}.

The covariance operator can be estimated similarly using finite sample average.
Given $m$ pairs of training examples $\Dcal_{XY}=\cbr{(x_1^i,x_2^i)}_{i\in [m]}$ drawn~\iid~from $\PP(X_1,X_2)$,
\begin{align}
 \widehat \Ccal_{X_1X_2} =\frac{1}{m}\sum_{i=1}^m \phi(x_1^i) \otimes \phi(x_2^i). \label{eq:empirical_covariance}
\end{align}
Similarly, given sample from distribution $\PP(X_1,X_2,X_3)$, one can estimate $\widehat \Ccal_{X_{1:3}} = \frac{1}{m} \sum_{i = 1}^m \phi(x_1^i) \otimes \phi(x_2^i) \otimes \phi(x_3^i)$.

By virtue of the kernel trick, most of the computation required for subsequent statistical inference using kernel embeddings can be reduced to the Gram matrix manipulation. The entries in the Gram matrix $K$ correspond to the kernel value between data points $x^i$ and $x^j$,~\ie,~$K_{ij} = k(x^i,x^j)$, and therefore its size is determined by the number of data points in the sample. The size of the Gram matrix is in general much smaller than the dimension of the feature spaces (which can be infinite). This enables efficient nonparametric methods using the kernel embedding representation. If the sample size is large, the computation in kernel embedding methods may be expensive. In this case, a popular solution is to use a low-rank approximation of the Gram matrix, such as incomplete Cholesky factorization~\cite{FinSch01}, which is known to work very effectively in reducing computational cost of kernel methods, while maintaining the approximation accuracy.

\section{Multi-View Latent Variable Models}

Multi-view latent variable models studied in this paper are a
special class of Bayesian networks in which
\begin{itemize}[noitemsep,nolistsep]
  \item observed variables $X_1, X_2, \ldots, X_\ell$ are conditionally independent given a {\bf discrete} latent variable $H$, and
	\item the conditional distributions, $\PP(X_t|H)$, of the $X_t, t \in [\ell]$ given the hidden variable $H$ can be different.
\end{itemize}
The conditional independent structure of a multi-view latent variable model is illustrated in Figure~\ref{fig:graphical-model}(a), and many complicated graphical models, such as the hidden Markov model in Figure~\ref{fig:graphical-model}(b), can be reduced to a multi-view latent variable model. {\bf For simplicity of exposition, we will explain our method using the model with symmetric view}. That is the conditional distribution are the same for each view,~\ie, $\PP(X|h) = \PP(X_1|h) =  \PP(X_2|h) =\PP(X_3|h)$. In Appendix~\ref{sec:symmetrization}, we will show that multi-view models with different views can be reduced to ones with symmetric view.

\begin{figure}[t]
\begin{center}
\subfigure[Na\"ive Bayes model]{\label{fig:exchangeable}
\begin{tikzpicture}
 [
   scale=0.85,
   observed/.style={circle,minimum size=0.7cm,inner sep=0mm,draw=black,fill=black!20},
   hidden/.style={circle,minimum size=0.7cm,inner sep=0mm,draw=black},
 ]
 \node [hidden,name=h] at ($(0,0)$) {$H$};
 \node [observed,name=x1] at ($(-1.5,-1)$) {$X_1$};
 \node [observed,name=x2] at ($(-0.5,-1)$) {$X_2$};
 \node at ($(0.5,-1)$) {$\dotsb$};
 \node [observed,name=xl] at ($(1.5,-1)$) {$X_\ell$};
 \draw [->] (h) to (x1);
 \draw [->] (h) to (x2);
 \draw [->] (h) to (xl);
\end{tikzpicture}}
\hfil\subfigure[Hidden Markov model]{\label{fig:hmm}
\begin{tikzpicture}
 [
   scale=0.85,
   observed/.style={circle,minimum size=0.7cm,inner sep=0mm,draw=black,fill=black!20},
   hidden/.style={circle,minimum size=0.7cm,inner sep=0mm,draw=black},
 ]
 \node [hidden,name=h1] at ($(-1.2,0)$) {$H_1$};
 \node [hidden,name=h2] at ($(0,0)$) {$H_2$};
 \node [name=hd] at ($(1.2,0)$) {$\dotsb$};
 \node [hidden,name=hl] at ($(2.4,0)$) {$H_\ell$};
 \node [observed,name=x1] at ($(-1.2,-1)$) {$X_1$};
 \node [observed,name=x2] at ($(0,-1)$) {$X_2$};
 \node [observed,name=xl] at ($(2.4,-1)$) {$X_\ell$};
 \draw [->] (h1) to (h2);
 \draw [->] (h2) to (hd);
 \draw [->] (hd) to (hl);
 \draw [->] (h1) to (x1);
 \draw [->] (h2) to (x2);
 \draw [->] (hl) to (xl);
\end{tikzpicture}}
\end{center}
\vspace{-5mm}
\caption{Examples of multi-view latent variable models.
}
\label{fig:graphical-model}
\vspace{-3mm}
\end{figure}
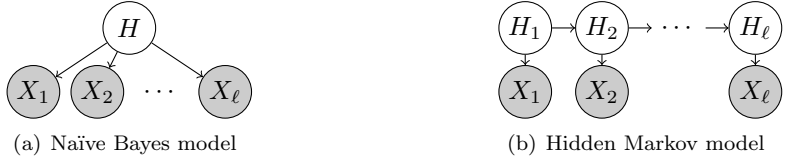

\subsection{Conditional Embedding Operator}

For simplicity of exposition, we  focus on a simple model with three observed variables, \ie, $\ell=3$. Suppose $H\in[k]$, then we can embed each conditional distribution $\PP(X|h)$ corresponding to a particular value of $H=h$ into the RKHS as
\begin{align}
  \mu_{X|h} = \int_{\Xcal} \phi(x)\, \PP(dx | h).
\end{align}
If we vary the value of $H$, we  obtain the kernel embedding for different $\PP(X|h)$. Conceptually, we can collect these embeddings into a matrix (with potentially infinite number of rows)
\begin{align}
  \Ccal_{X|H} = \rbr{\mu_{X|h=1},\mu_{X|h=2},\ldots,\mu_{X|h=k}},
\end{align}
which is called the conditional embedding operator. If we use the standard basis $e_h$ in $\RR^k$ to represent each value of $h$, we can retrieve each $\mu_{X|h}$ from $\Ccal_{X|H}$ by
\begin{align}
  \mu_{X|h} = \Ccal_{X|H} e_h
\end{align}
Once we have the conditional embedding $\mu_{X|h}$, we can estimate the density $p(x|h)$ by performing an inner product $p(x|h) = \inner{\phi(x)}{\mu_{X|h}}$.

\subsection{Factorized Kernel Embedding}

Then the distributions, $\PP(X_1,X_2)$ and $\PP(X_1,X_2,X_3)$, can be factorized respectively
\begin{align*}
	\PP(dx_1, dx_2) &= \int_{\Hcal}\, \PP(dx_1 | h )\, \PP(d x_2 | h)\, \PP(dh),~\text{and}\\
	\PP(dx_1, dx_2, dx_3) &= \int_{\Hcal}\, \PP(dx_1 | h )\, \PP(dx_2 | h)\, \PP(dx_3 | h)\, \PP(dh).
\end{align*}
Since we assume the hidden variable $H \in [k]$ is discrete,
we let $\pi_h:=\PP(h)$. Furthermore, if we apply Kronecker delta kernel $\delta(h,h')$ with feature map $e_h$, then the embeddings for $\PP(H)$
\begin{align*}
 \Ccal_{HH} &= \EE_H[e_H \otimes e_H] = \rbr{
  \begin{array}{ccc}
    \pi_1& \ldots & 0 \cr
    \vdots & \ddots & \vdots \cr
    0 & \ldots & \pi_k
  \end{array}
 },
 ~\text{and}~\\
 \Ccal_{HHH} &= \EE_H[e_H \otimes e_H \otimes e_H]  \\
	&= \rbr{
  \begin{array}{c}
    \cr
    \pi_h\ \delta(h,h')\ \delta(h',h'') \cr
    \cr
  \end{array}
 }_{h,h',h''\in[k]}
\end{align*}
are diagonal tensors. Making use of $\Ccal_{HH}$ and $\Ccal_{HHH}$, and the factorization of the distributions $\PP(X_1,X_2)$ and $\PP(X_1,X_2,X_3)$, we obtain the factorization of the embedding of
$\PP(X_1,X_2)$ (second order embedding)
\begin{align}
  &\Ccal_{X_1X_2} \nonumber \\
  &= \int_{\Hcal} \rbr{\int_\Xcal \phi(x_1)\, \PP(dx_1|h)} \otimes \rbr{\int_\Xcal \phi(x_2)\, \PP(dx_2 |h)} \PP(dh) \nonumber \\
  &= \int_{\Hcal} \rbr{\Ccal_{X|H} e_h} \otimes \rbr{\Ccal_{X|H} e_h}\, \PP(dh) \nonumber \\
  &= \Ccal_{X|H}\, \rbr{\int_{\Hcal} e_h \otimes e_h\, \PP(dh)}\, \Ccal_{X|H}^\top \nonumber \\
  &= \Ccal_{X|H}\, \Ccal_{HH}\, \Ccal_{X|H}^\top, \label{eq:pair_factorization}
\end{align}
and that of $\PP(X_1,X_2,X3)$ (third order embedding)
\begin{align}
  &\Ccal_{X_1 X_2 X_3}
  = \Ccal_{HHH} \times_1 \Ccal_{X|H} \times_2 \Ccal_{X|H} \times_3 \Ccal_{X|H}. \label{eq:triple_factorization}
\end{align}

\subsection{Identifiability of Parameters}

We note that $\Ccal_{X|H} = \rbr{\mu_{X|h=1},\ \mu_{X|h=2},\ \ldots,\ \mu_{X|h=k}}$, and the kernel embeddings for $\Ccal_{X_1 X_2}$ and $\Ccal_{X_1 X_2 X_3}$ can be alternatively written as
\begin{align}
	\Ccal_{X_1 X_2}
  & = \sum_{h \in [k]} \pi_h\cdot \mu_{X|h} \otimes \mu_{X|h}, \label{eq:joint2} \\
  \Ccal_{X_1 X_2 X_3}
  &= \sum_{h \in [k]} \pi_h\cdot \mu_{X|h} \otimes \mu_{X|h} \otimes \mu_{X|h} \label{eq:joint3}
\end{align}
Allman et al.~\cite{Allman09} showed that, under mild conditions, a finite mixture of nonparametric product distributions is identifiable. The multi-view latent variable model in~\eq{eq:joint3} has the same form as a finite mixture of nonparametric product distribution, and therefore we can adapt Allman's results to the current setting.
\begin{proposition}[Identifiability]\label{prop:identifiability}
\vspace{-2mm}
  Let $\PP(X_1,X_2,X_3)$ be a multi-view latent variable model, such that the conditional distributions $\cbr{\PP(X|h)}_{h \in [k]}$ are linearly independent. Then, the set of parameters $\cbr{\pi_h, \mu_{X|h}}_{h \in [k]}$ are identifiable from $\Ccal_{X_1 X_2 X_3}$, up to label swapping of the hidden variable $H$.
\vspace{-2mm}
\end{proposition}

{\bf Example 1.} The probability vector of a discrete variable $X \in [n]$, and the joint probability table of two discrete variables $X_1 \in [n]$ and $X_2 \in [n]$, are both kernel embeddings. To see this, let the kernel be the Kronecker delta kernel $k(x,x') = \delta(x,x')$ whose feature map $\phi(x)$ is the standard basis of $e_{x}$ in $\RR^n$. The $x$-th dimension of $e_{x}$ is 1 and 0 otherwise. Then
\begin{align}
    \mu_X
		& = \rbr{
      \begin{array}{ccc}
         \PP(x = 1) &
         \hdots &
         \PP(x = n)
       \end{array}
    }^\top, \nonumber \\
		\Ccal_{X_1X_2}
		&=
		\rbr{
        \begin{array}{c}
            \cr
            \PP(x_1=s,x_2=t) \cr
						\cr
        \end{array}
    }_{s,t \in [n]}. \nonumber
\end{align}
We require that the conditional probability table $\{P(X|h)\}_{h\in [k]}$ to  have full column rank for identifiability in this case.

{\bf Example 2.} Suppose we have a $k$-component mixture of one dimensional  spherical Gaussian distributions. The Gaussian components have identical covariance  $\sigma^2$, but their mean values are distinct. Note that this model is not identifiable under the framework of~\cite{Hsu13} since the mean values are just scalars and therefore, rank deficient. However, if we embed the density functions using universal kernels such as Gaussian RBF kernel, it can be shown that the mixture model becomes identifiable. This is because we are working with the entire density function which are linearly independent from each other. Thus, under a non-parametric framework, we can incorporate overcomplete mixtures, where the number of components can exceed the observed dimensionality.

Finally, we remark that the identifiability result in Proposition~\ref{prop:identifiability} can be extended to cases where the conditional distributions do not satisfy linear independence, e.g.~\cite{Kruskal:77,DeLathauwerEtal:FOOBI,AnandkumarEtal:overcomplete13}. However, in general, it is not tractable to learn such  models and we do not consider them here.

\section{Kernel Algorithm}

We first design a kernel algorithm to recover the parameters, $\cbr{\pi_h, \mu_{X|h}}_{h \in [k]}$, of the multi-view latent variable model based on $\Ccal_{X_1 X_2}$ and $\Ccal_{X_1 X_2 X_3}$. This can be easily extended to the sample versions and this is discussed in Section~\ref{sec:sample}. For clarity of the presentation, we first present the symmetric view case, and then, extend to more general version.

\subsection{Population Case}

We  first derive the algorithm for the population case as if we could access the true operator $\Ccal_{X_1 X_2}$ and $\Ccal_{X_1 X_2 X_3}$. Its finite sample counterpart  will be presented in the next section. The algorithm can be thought of as a kernel generalization of the algorithm in~\cite{AnandkumarEtal:community12} using embedding representations.

{\bf Step 1.} We perform eigen-decomposition of $\Ccal_{X_1 X_2}$,
$$\Ccal_{X_1 X_2} = \sum_{i=1}^{\infty} \sigma_i \cdot u_i \otimes u_i$$
where the eigen-values are ordered in non-decreasing manner.
According to the factorization in Eq.~\eq{eq:pair_factorization}, $\Ccal_{X_1 X_2}$ has rank $k$.
Let the leading eigenvectors corresponding to the largest $k$ eigen-value be  $\Ucal_k:=(u_1,u_2,\ldots,u_k)$, and the eigen-value matrix be $S_k:=\diag(\sigma_1,\sigma_2,\ldots,\sigma_k)$. We define the whitening operator $\Wcal:= \Ucal_k S_k^{-1/2}$ which satisfies
\begin{align*}
  \Wcal^\top \Ccal_{X_1X_2} \Wcal = (\Wcal^\top \Ccal_{X|H} \Ccal_{HH}^{1/2}) (\Ccal_{HH}^{1/2} \Ccal_{X|H}^\top \Wcal) = I,
\end{align*}
and $M:=\Wcal^\top \Ccal_{X|H} \Ccal_{HH}^{1/2}$ is an orthogonal matrix.

{\bf Step 2.} We apply the whiten operator to the 3rd order kernel embedding $\Ccal_{X_1 X_2 X_3}$
$$
  \Tcal := \Ccal_{X_1 X_2 X_3} \times_1 (\Wcal^\top) \times_2 (\Wcal^\top) \times_3 (\Wcal^\top).
$$
According to the factorization in Eq.~\eq{eq:triple_factorization},
$$
  \Tcal = \Ccal_{HHH}^{-1/2} \times_1 M \times_2 M \times_3 M,
$$
which is a tensor with orthogonal factors. Essentially, each column $v_i$ of $M$ is an eigenvector of the tensor $\Tcal$.

{\bf Step 3.} We use tensor power method to find eigenvectors $M$ for $\Tcal$~\cite{AnandkumarEtal:community12}. We provide the method in the Appendix in Algorithm~\ref{alg:robustpower} for completeness.

{\bf Step 4.} We recover the conditional embedding operator by undoing the whitening step
$$
  \Ccal_{X|H} = (\mu_{X|h=1},\mu_{X|h=1},\ldots,\mu_{X|h=k}) = (\Wcal)^\dagger M.
$$

\subsection{Finite Sample Case}\label{sec:sample}

Given $m$ observation $\Dcal_{X_1 X_2 X_3}=\{(x_1^i,x_2^i,x_3^i)\}_{i \in [m]}$ drawn~\iid~from a multi-view latent variable model $\PP(X_1,X_2,X_3)$, we now design a kernel algorithm to estimate the latent parameters from data. Although the empirical kernel embeddings can be infinite dimensional, we can carry out the decomposition using just the kernel matrices.
We  denote the implicit feature matrix by
\begin{align*}
  \Phi &:= (\phi(x_1^1), \ldots, \phi(x_1^m), \phi(x_2^1),  \ldots, \phi(x_2^m)),  \\
  \Psi &:= (\phi(x_2^1), \ldots, \phi(x_2^m), \phi(x_1^1),  \ldots, \phi(x_1^m)),
\end{align*}
and the corresponding kernel matrix by $K = \Phi^\top \Phi$ and $L = \Psi^\top \Psi$ respectively.
Then the steps in the population case can be mapped one-by-one into kernel operations.

{\bf Step 1.} We  perform a kernel eigenvalue decomposition of the empirical 2nd order embedding
$$
  \widehat \Ccal_{X_1 X_2}:= \frac{1}{2m} \sum_{i=1}^{m} \rbr{\phi(x_1^i) \otimes \phi(x_2^i) + \phi(x_2^i) \otimes \phi(x_1^i)},
$$
which can be expressed succinctly as $\widehat \Ccal_{X_1 X_2} = \frac{1}{2m} \Phi \Psi^\top$.
Its leading $k$ eigenvectors $\widehat \Ucal_k = (\widehat u_1,\ldots,\widehat u_k)$  lie in the span of the column of  $\Phi$,~\ie,~$\widehat \Ucal_k = \Phi (\beta_1,\ldots,\beta_k)$ with $\beta \in \RR^{2m}$. Then we can transform the eigen-value decomposition problem for an infinite dimensional matrix to a problem involving finite dimensional kernel matrices,
\begin{align*}
	\widehat \Ccal_{X_1 X_2}\, \widehat \Ccal_{X_1 X_2}^\top\, u = \widehat \sigma^2 \;u
	&~\Rightarrow~
	\frac{1}{4m^2}\Phi \Psi^\top \Psi \Phi^\top \Phi \beta = \widehat \sigma^2 \,\Phi \beta \\
	&~\Rightarrow~
	\frac{1}{4m^2} K L K \beta = \widehat \sigma^2 \,K \beta.
\end{align*}
Let the Cholesky decomposition of $K$ be $R^\top R$. Then by redefining $\widetilde{\beta}=R\beta$, and solving an eigenvalue problem
\begin{align}
 \frac{1}{4m^2} R L R^\top \widetilde{\beta} =\widehat  \sigma^2 \, \widetilde{\beta},~~\text{and obtain}~\beta = R^{\dagger} \widetilde{\beta}.
\end{align}
The resulting eigenvectors satisfy $u_i^\top u_{i'} = \beta_i^\top \Phi^\top \Phi \beta_{i'} =  \beta_{i}^\top K  \beta_{i'} =  \widetilde{\beta}_{i}^\top \widetilde{\beta}_{i'}=\delta_{ii'}$.
This step is summarized in Algorithm~\ref{alg:svd}.

\begin{algorithm}[t!]
\caption{KernelSVD($K$, $L$, $k$)}
	\textbf{Out}: $\widehat S_k$ and $(\beta_1,\ldots,\beta_k)$\\[-0.4cm]
  \begin{algorithmic}[1]
    \STATE Cholesky decomposition:\ $K=R^\top R$
    \STATE Eigen-decomposition:\ $\frac{1}{4m^2} R L R^\top \widetilde{\beta} = \widehat \sigma^2\,\widetilde{\beta}$
    \STATE Use $k$ leading eigenvalues:\ $\widehat S_k = \diag(\widehat \sigma_1,\ldots,\widehat \sigma_k)$
    \STATE Use $k$ leading eigenvectors:\ $(\widetilde{\beta}_1,\ldots,\widetilde{\beta}_k)$ to
    compute:\ $(\beta_1,\ldots,\beta_k) = R^\dagger (\widetilde{\beta}_1,\ldots,\widetilde{\beta}_k)$
  \end{algorithmic}
  \label{alg:svd}
\end{algorithm}

{\bf Step 2.} We whiten the empirical 3rd order embedding
\begin{align*}
  &\widehat \Ccal_{X_1 X_2 X_3}:= \frac{1}{3m}\sum_{i=1}^{m} (\phi(x_1^i) \otimes \phi(x_2^i) \otimes \phi(x_3^i) \\
  &+ \phi(x_3^i) \otimes \phi(x_1^i) \otimes \phi(x_2^i) + \phi(x_2^i) \otimes \phi(x_3^i) \otimes \phi(x_1^i))
\end{align*}
using $\widehat \Wcal:= \widehat \Ucal_k \widehat S_k^{-1/2}$, and obtain
\begin{align*}
  &\widehat \Tcal := \frac{1}{3m}\sum_{i=1}^m (\xi(x_1^i) \otimes \xi(x_2^i) \otimes \xi(x_3^i) \\
  &+ \xi(x_3^i) \otimes \xi(x_1^i) \otimes \xi(x_2^i) + \xi(x_2^i) \otimes \xi(x_3^i) \otimes \xi(x_1^i)),
\end{align*}
where
$$
	\xi(x_1^i) := \widehat S_k^{-1/2} (\beta_1,\ldots,\beta_k)^\top \Phi^\top \phi(x_1^i)~ \in~\RR^k.
$$

{\bf Step 3.} We run tensor power method~\cite{AnandkumarEtal:community12} on the finite dimension tensor $\widehat \Tcal$ to obtain its leading $k$ eigenvectors $\widehat M:=(\widehat v_1,\ldots,\widehat v_k)$.

{\bf Step 4.} The estimates of the conditional embeddings are
\begin{align*}
  \widehat \Ccal_{X|H} = (\widehat \mu_{X|h=1},\ldots, \widehat \mu_{X|h=k}) = \Phi (\beta_1,\ldots,\beta_k) \widehat  S_k^{1/2} \widehat M.
\end{align*}

\subsection{Symmetrization}
\label{sec:symmetrization}

In this section, we will extend the algorithm to the general case where the conditional distributions for each view are different. Without loss of generality, we will consider recover the operator $\mu_{X_3|h}$ for conditional distribution $\PP(X_3|h)$. The same strategy applies to other views. The idea is to reduce the multi-view case to the identical-view case based on a method by~\cite{AnandkumarEtal:twosvd12}.

Given the observations $\mathcal{D}_{X_1X_2X_3}=\{(x_1^i, x_2^i, x_3^i)\}_{i\in[m]}$ drawn \emph{i.i.d.} from a multi-view latent variable model $\mathbb{P}(X_1, X_2, X_3)$, let the kernel matrix associated with $X_1$, $X_2$ and $X_3$ be $K$, $L$ and $G$ respectively and the corresponding feature map be $\phi$, $\psi$ and $\upsilon$ respectively. Furthermore, let the corresponding feature matrix be $\widetilde \Phi=(\phi(x_1^1),\ldots,\phi(x_1^m))$, $\widetilde\Psi=(\phi(x_2^1),\ldots,\phi(x_2^m))$ and $\widetilde \Upsilon=(\phi(x_3^1),\ldots,\phi(x_3^m))$. Then, we have the empirical estimation of the second/third-order embedding as
\begin{align*}
&\widehat\Ccal_{X_1 X_2} = \frac{1}{m}\widetilde \Phi \widetilde \Psi^\top,~\widehat\Ccal_{X_3 X_1} = \frac{1}{m}\widetilde \Upsilon \widetilde \Phi^\top,~\widehat\Ccal_{X_2 X_3} = \frac{1}{m} \widetilde \Psi \widetilde \Upsilon^\top\\
&\widehat\Ccal_{X_1 X_2 X_3}:=
\frac{1}{m}\bm{I}_n \times_1 \widetilde \Phi \times_2 \widetilde \Psi \times_3 \widetilde \Upsilon
\end{align*}

Find two arbitrary matrices $\bm{A,B}\in \mathbb{R}^{k \times \infty}$, so that $\bm{A}\widehat{\mathcal{C}}_{X_1X_2}\bm{B}^\top$ is invertible. Theoretically, we could randomly select $k$ columns from $\Phi$ and $\Psi$ and set $\bm{A} = \Phi_k^\top, \bm{B} = \Psi_k^\top$. In practial, the first $k$ leading eigenvector directions of respect \emph{RKHS} works better.
Then, we have
\begin{eqnarray*}
\widetilde{\mathcal{C}}_{X_1 X_2} &=& \frac{1}{m}\widetilde \Phi_k^\top \widetilde\Phi\widetilde\Psi^\top\widetilde\Psi_k = \frac{1}{m}{K}_{nk}^\top{L}_{nk}\\
\widetilde{\mathcal{C}}_{X_3 X_1} &=& \widehat{\mathcal{C}}_{X_3X_1}\widetilde\Phi_k = \frac{1}{m}\widetilde\Upsilon{K}_{nk}\\
\widetilde{\mathcal{C}}_{X_3 X_2} &=& \widehat{\mathcal{C}}_{X_3X_2}\widetilde\Psi_k = \frac{1}{m}\widetilde\Upsilon{L}_{nk}\\
\widetilde{\mathcal{C}}_{X_1 X_2 X_3} &=&
\widehat{\mathcal{C}}_{X_1 X_2 X_3}\times_1\widetilde\Phi_k^\top
\times_2\widetilde\Psi_k^\top = \frac{1}{m} \bm{I}_n \times_1
{K}_{nk}^\top \times_2 {L}_{nk}^\top \times_3
\widetilde\Upsilon
\end{eqnarray*}

Based on these matrices, we could reduce to a single view
\begin{eqnarray*}
Pair_3 &=&
\widetilde{\mathcal{C}}_{X_3X_1}(\widetilde{\mathcal{C}}_{X_1X_2}^\top)^{-1}\widetilde{\mathcal{C}}_{X_3X_2}\\
&=&\frac{1}{m}\widetilde\Upsilon{K}_{nk}({L}_{nk}^\top{K}_{nk})^{-1}{L}_{nk}^\top\widetilde\Upsilon^\top = \frac{1}{m}\widetilde\Upsilon{H}\widetilde\Upsilon^\top
\end{eqnarray*}
where ${H} = {K}_{nk}(\mathcal{L}_{nk}^\top{K}_{nk})^{-1}{L}_{nk}^\top$.

Assume the leading $k$ eigenvectors $\nu_k$ lie in the span of the column of $\Upsilon$, i.e., $\nu_k = \Upsilon \beta_k$ where $\beta_k\in \mathbb{R}^{m\times 1}$
\begin{eqnarray*}
Pair_3\nu = \lambda \nu &\Rightarrow& (Pair_3)^\top Pair_3\nu = \lambda^2 \nu \\
&\Rightarrow&
\frac{1}{m^2} \widetilde\Upsilon{H}^\top\widetilde\Upsilon^\top\widetilde\Upsilon{H}\widetilde\Upsilon^\top\nu
= \lambda^2\nu \\
&\Rightarrow&
\frac{1}{m^2}\widetilde\Upsilon{H^\top GHG}\bm{\beta} =
\lambda^2 \widetilde\Upsilon\bm{\beta} \\
&\Rightarrow& \frac{1}{m^2}{GH^\top GHG}\beta
= \lambda^2{G}\beta
\end{eqnarray*}
Then, we symmetrize and whiten the third-order embedding
\begin{eqnarray}
Triple_3 = \frac{1}{m}\widetilde{\mathcal{C}}_{X_1X_2X_3} \times_1
[\widetilde{\mathcal{C}}_{X_3X_2}\widetilde{\mathcal{C}}_{X_1X_2}^{-1}]
\times_2
[\widetilde{\mathcal{C}}_{X_3X_1}\widetilde{\mathcal{C}}_{X_2X_1}^{-1}]
\end{eqnarray}
Plug
$\widetilde{\mathcal{C}}_{X_3X_2}\widetilde{\mathcal{C}}_{X_1X_2}^{-1} =
\widetilde\Upsilon{L}_{nk}({K}_{nk}^\top{L}_{nk})^{-1}$
and
$\widetilde{\mathcal{C}}_{X_3X_1}\widetilde{\mathcal{C}}_{X_2X_1}^{-1} =
\widetilde\Upsilon{K}_{nk}({L}_{nk}^\top{K}_{nk})^{-1}$,
we have

\begin{eqnarray*}
Triple_3  = \frac{1}{m}\bm{I}_n \times_1
\widetilde\Upsilon{L}_{nk}({K}_{nk}^\top
{L}_{nk})^{-1}{K}_{nk}^\top\\ \times_2
\widetilde\Upsilon{K}_{nk}({L}_{nk}^\top
{K}_{nk})^{-1}{L}_{nk}^\top \times_3 \Upsilon
\end{eqnarray*}

We multiply each mode with $\Upsilon \beta \widehat{S}_k^{-\frac{1}{2}}$ to
whitening the data and apply power method to decompose it
\begin{eqnarray*}
\widehat{\mathcal{T}} &=& Triple_3 \times_1 \widehat{S}_k^{-\frac{1}{2}}\beta^\top\widetilde\Upsilon^\top \times_2
\widehat{S}_k^{-\frac{1}{2}}\beta^\top\widetilde\Upsilon^\top \times_3 \widehat{S}_k^{-\frac{1}{2}}\beta^\top\widetilde\Upsilon^\top\\
&=& \frac{1}{m}\bm{I}_n \times_1
\widehat{S}_k^{-\frac{1}{2}}\beta^\top{G}\mathcal{L}_{nk}({K}_{nk}^\top
{L}_{nk})^{-1}{K}_{nk}^\top \times_2\\
&&\widehat{S}_k^{-\frac{1}{2}}\beta^\top{G}{K}_{nk}({L}_{nk}^\top
{K}_{nk})^{-1}{L}_{nk}^\top \times_3
\widehat{S}_k^{-\frac{1}{2}}\beta^\top{G}
\end{eqnarray*}
Apply the algorithm for symmetric case in previous section to $\widehat{\mathcal{T}}$, we could recover the conditional distribution operator.
\section{Sample Complexity}

Let $\rho:=\sup_{x \in \Xcal} k(x,x)$,   $\| \cdot\|_{}$ be the Hilbert-Schmidt norm, $\pi_{\min}:=\min_{i\in [k]} \pi_i$ and $\sigma_k(\Ccal_{X_1X_2})$ be the $k$-th largest singular value of $\Ccal_{X_1X_2}$.

\begin{theorem}[Sample Bounds]\label{thm:samplebound}
Pick  any $\delta\in (0,1)$. When the number of samples $m$ satisfies
\[ m >\frac{\theta\rho^2  \log\frac{\delta}{2}}{\sigma^2_k(\Ccal_{X_1, X_2})},
\quad \theta:= \max\left(\frac{C_3 k^2 \rho}{\sigma_k( \Ccal_{X_1, X_2})}, \frac{C_4k^{2/3}\iffalse(1+\sigma_{k+1}(\Ccal_{X_1, X_2}))^2 \fi }{\pi_{\min}^{1/3}}\right),\] for some constants $C_3, C_4>0$, and the number of iterations $N$  and  the number of random initialization vectors $L$  (drawn uniformly on the sphere $\mathcal{S}^{k-1}$)  satisfy
\begin{align*}
  N \geq C_2 \cdot \biggl( \log(k) + \log\log\Bigl(
 \frac{1}{\sqrt{\pi}_{\min}\epsilon_T} \Bigr) \biggr),
\end{align*}
for constant $C_2>0$ and  $L = \poly(k) \log(1/\delta)$,  the robust power method in~\cite{AnandkumarEtal:community12} yields eigen-pairs $(\h{\lambda}_i, \h{\phi}_i)$ such that there exists a permutation $\eta$, with probability $1-4\delta$, we have
\begin{align*}
&\|\pi^{-1/2}_{j} \mu_{X|h=j}-\h{\phi}_{\eta(j)}\| \leq 8 \epsilon_T \cdot\pi^{-1/2}_{j}
, \\
&|\pi^{-1/2}_{j}-\h{\lambda}_{\eta(j)}| \leq  5\epsilon_T, \quad \forall j \in [k]
,
\end{align*}
and
\[
\biggl\|
T - \sum_{j=1}^k \hat\lambda_j \hat{\phi}_j^{\otimes 3}
\biggr\| \leq 55\eps_T,
\] where $\eps_T$ is the tensor perturbation bound
\begin{align*} \eps_T := \|\h{\Tcal} - \Tcal\| \leq&
\frac{8 \rho^{1.5} \sqrt{\log\frac{\delta}{2}}}{\sqrt{m} \, \sigma_k^{1.5}(\Ccal_{X_1, X_2})} + \frac{512 \sqrt{2} \rho^3 \left(\log\frac{\delta}{2}\right)^{1.5}}{m^{1.5} \,\sigma_k^{3}(\Ccal_{X_1, X_2}) \sqrt{\pi}_{\min}}\end{align*}
\end{theorem}

Thus, the above result provides bounds on the estimated eigen-pairs using the robust tensor power method.
The proof is in Appendix~\ref{app:samplebound}.

\paragraph{Remarks: }We note that the sample complexity is  $\poly(k, \rho, 1/\pi_{\min}, 1/\sigma_k(\Ccal_{X_1, X_2}))$ of a low order, and in particular,  it is $O(k^2)$, when the other parameters are fixed. For the special case of discrete measurements, where the kernel $k(x,x')=\delta(x,x')$, we have $\rho=1$. Note that the sample complexity depends in this case only on the number of components $k$ and not on the dimensionality of the observed state space.   Thus, the robust tensor method has efficient sample and computational complexities for non-parametric latent variable estimation.

\section{Experiments}

{\bf Methods.} We compared our kernel nonparametric algorithm with three alternatives
\begin{enumerate}[noitemsep,nolistsep]
  \item The EM algorithm for mixture of Gaussians. The EM algorithm is not guaranteed to find the global solution in each trial. Thus we randomly initialize it $10$ times.
  \item The spectral algorithm for mixture of spherical Gaussians~\citep{Hsu13}. The assumption in~\citet{Hsu13} is very restrictive: the collection of spherical Gaussian centers need to span a $k$-dimension subpsace.
  \item A discretization based spectral algorithm~\citep{Hiroyuki10}. This algorithm approximates the joint distribution of the observed variables with histogram and then applies the spectral algorithm to recover the discretized conditional density. It is well-known that density estimation using histogram suffers from poor performance even for $3$-dimension data. The error of this algorithm is typically $10$ times larger than alternatives. To make the curves for other methods clearer, we did not plot the performance of \citet{Hiroyuki10} algorithm in the figures.
\end{enumerate}
Our method has a hyper-parameter, kernel bandwidth, which we selected for each view separately using cross-validation.

\subsection{Synthetic Data}

\subsubsection{General Case: Different Conditional Distributions}
We generated three-dimensional synthetic data from various mixture models. The variables corresponding to the dimensions are independent given the latent component indicator. More specifically, we explored two settings:
\begin{enumerate}[noitemsep,nolistsep]
  \item Gaussian conditional densities with different variances;
  \item Mixture of Gaussian and shifted Gamma conditional densities.
\end{enumerate}

The shifted Gamma distribution has density
$$p(x-\mu)=\frac{(x-\mu)^{(d-1)}e^{-x/\theta}}{\theta^d \Gamma(d)},~x\geq \mu$$
where we chose the shape parameter $d\le 1$ such that density is very skewed. Furthermore, we chose the mean and variance parameters of the Gaussian/Gamma density such that component pair-wise overlap is relatively small according to the Fisher ratio $\frac{(\mu_1 - \mu_2)^2}{\sigma_1^2+\sigma_2^2}$.

\begin{figure*}[!t]
  \hspace{-7mm}
  \renewcommand{\tabcolsep}{1pt}
  \begin{tabular}{cccc}
    \includegraphics[width=0.26\textwidth]{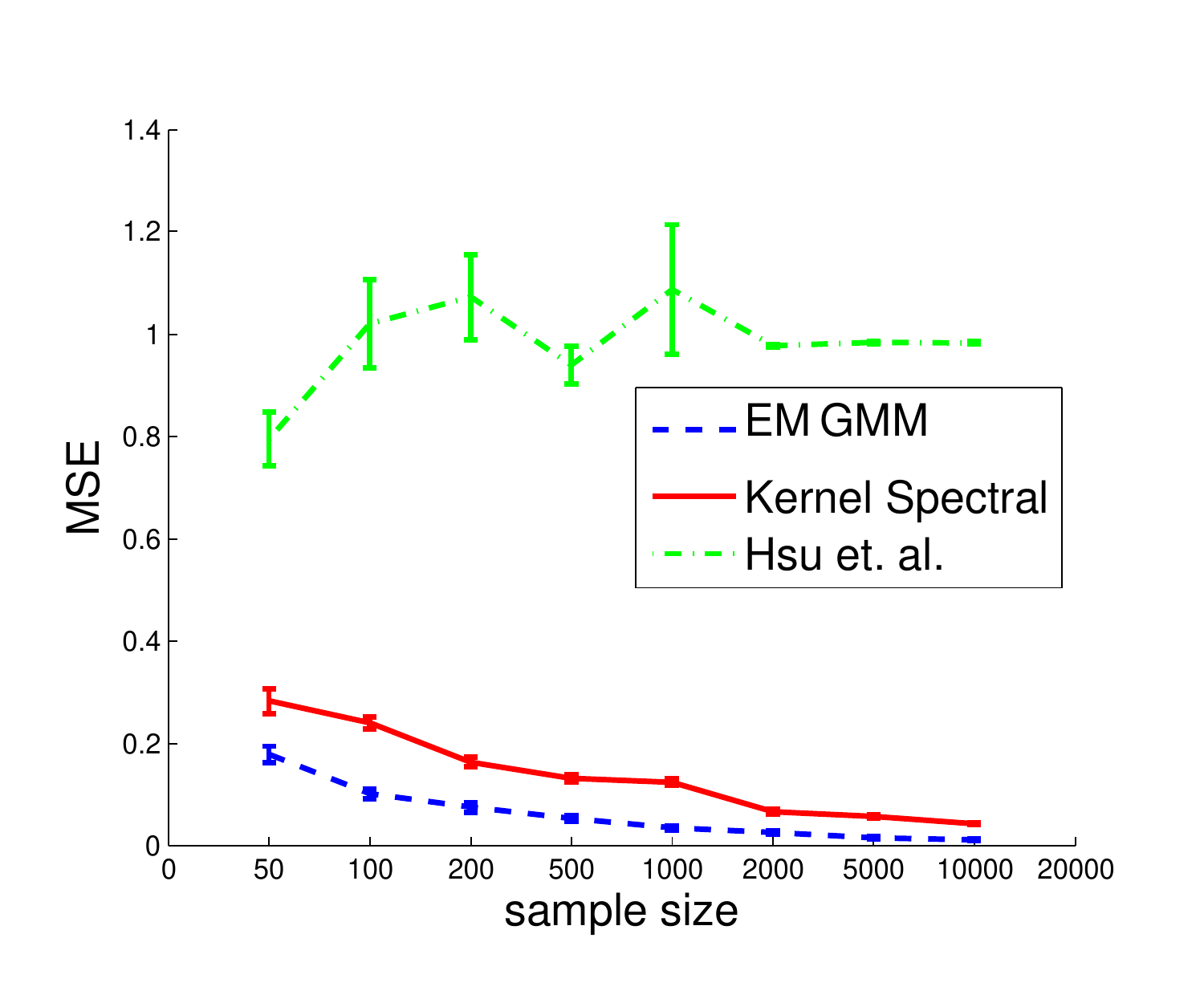} &
    \includegraphics[width=0.26\textwidth]{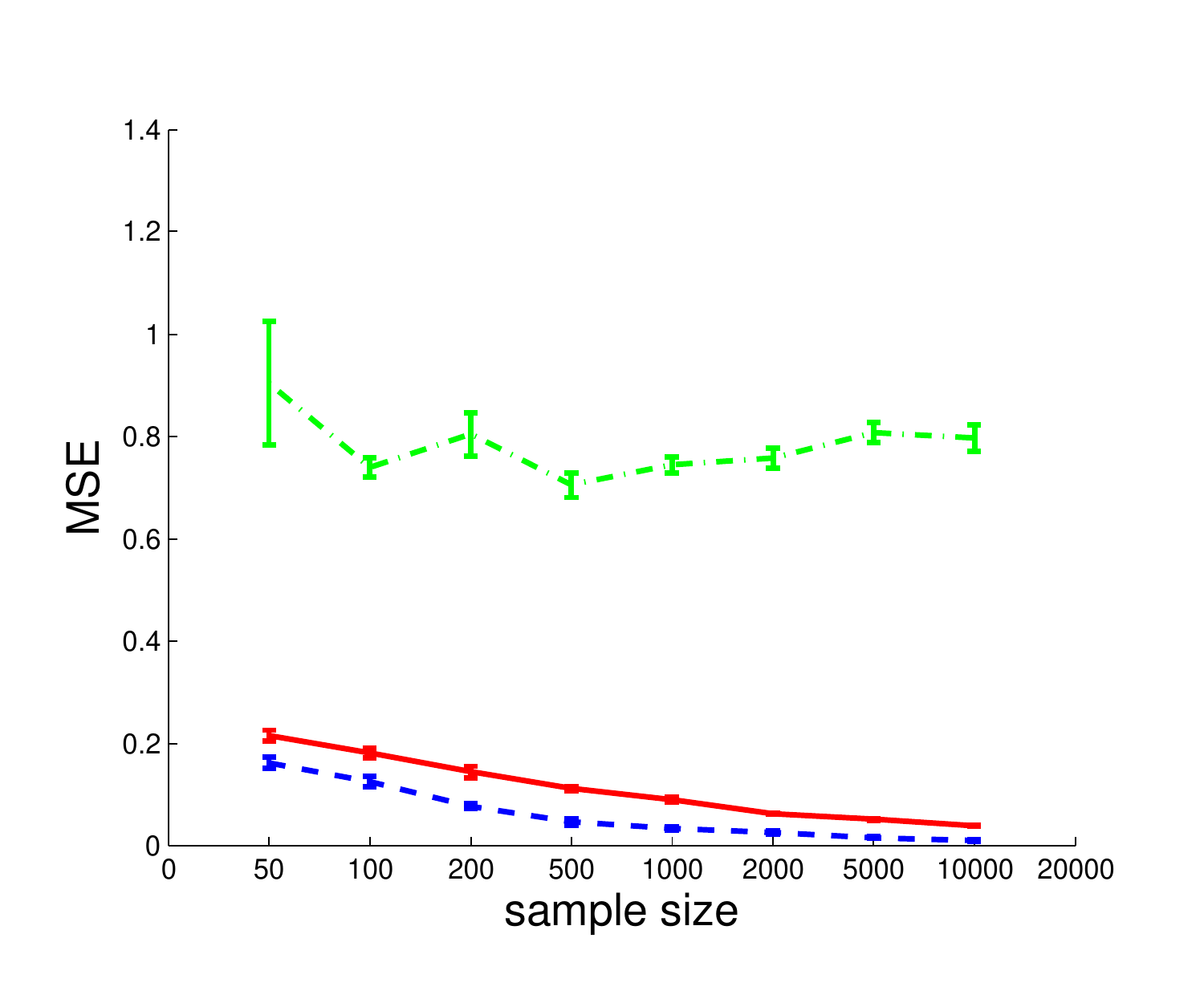} &
    \includegraphics[width=0.26\textwidth]{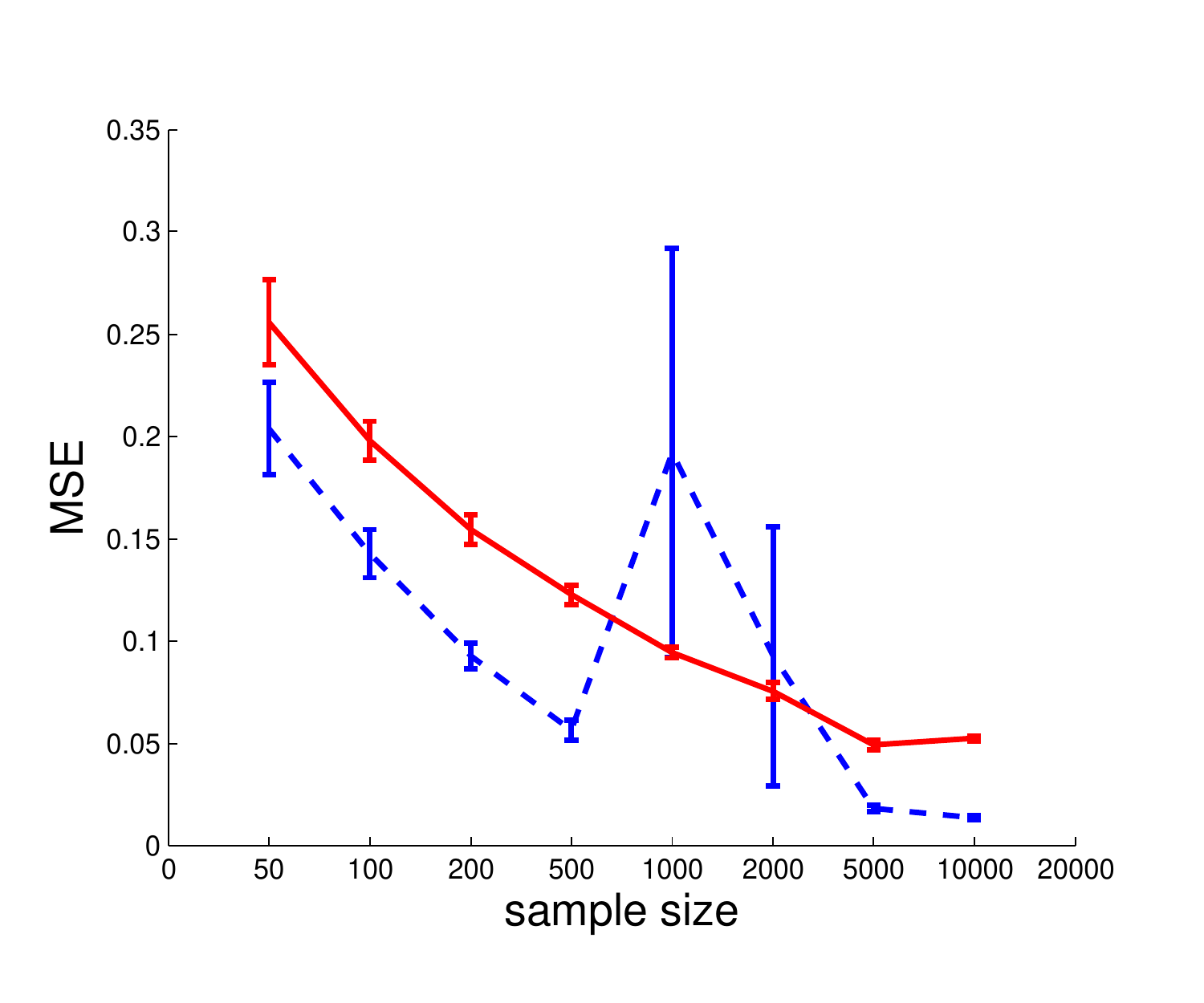} &
    \includegraphics[width=0.26\textwidth]{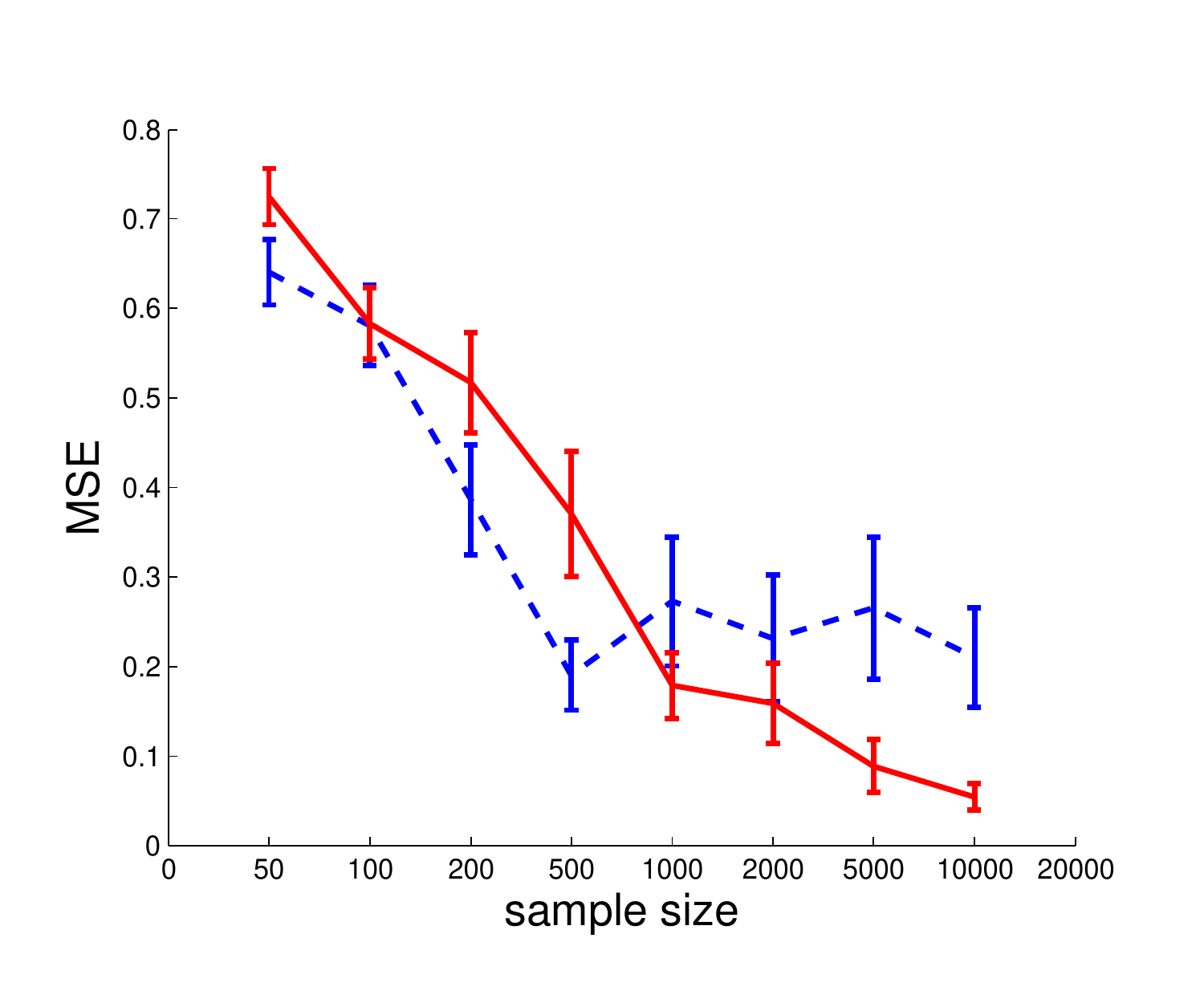} \\
    (a) Gaussian $k=2$ & (b) Gaussian $k=3$ & (c) Gaussian $k=4$ & (d) Gaussian $k=8$ \\
    \includegraphics[width=0.26\textwidth]{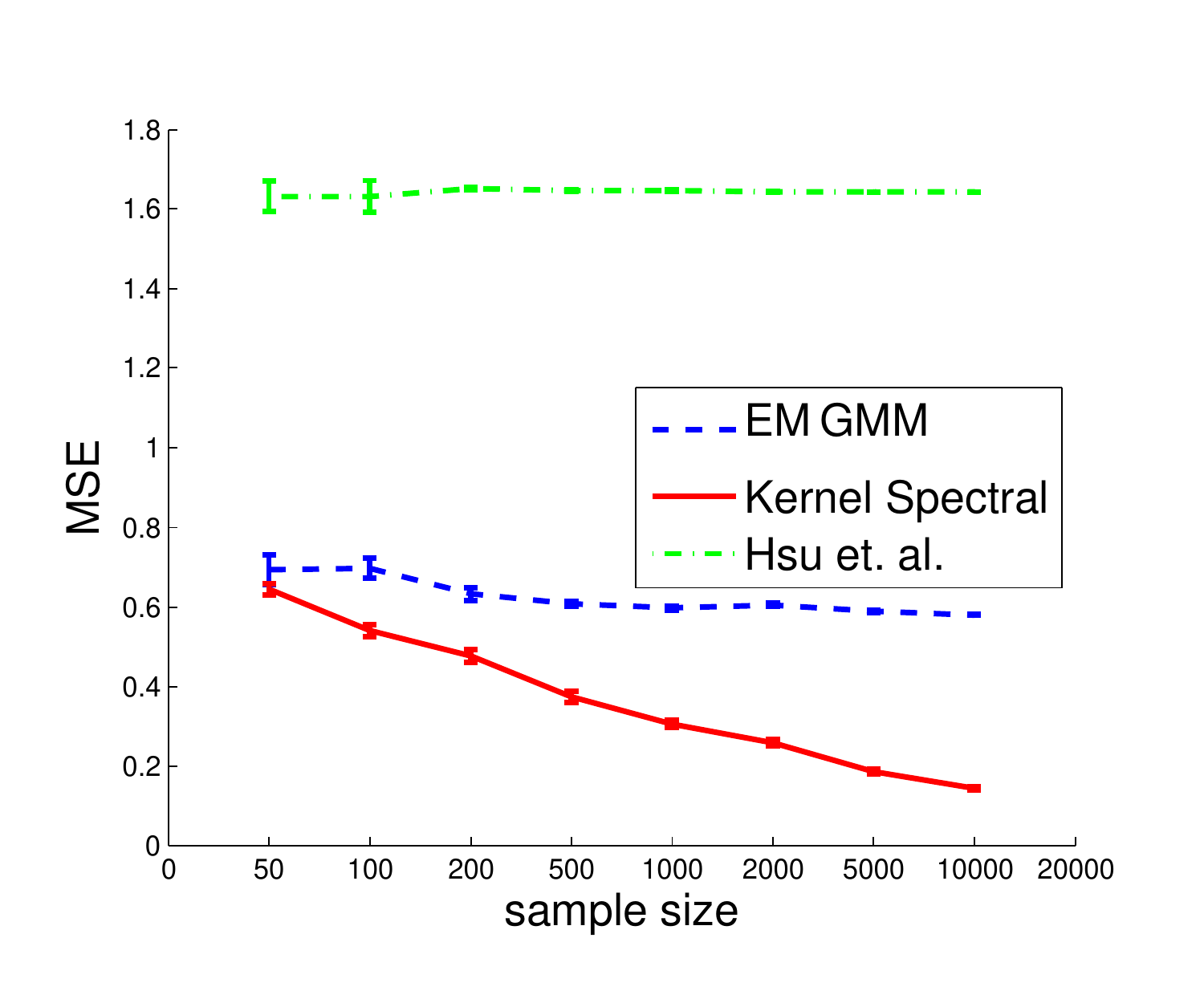} &
    \includegraphics[width=0.26\textwidth]{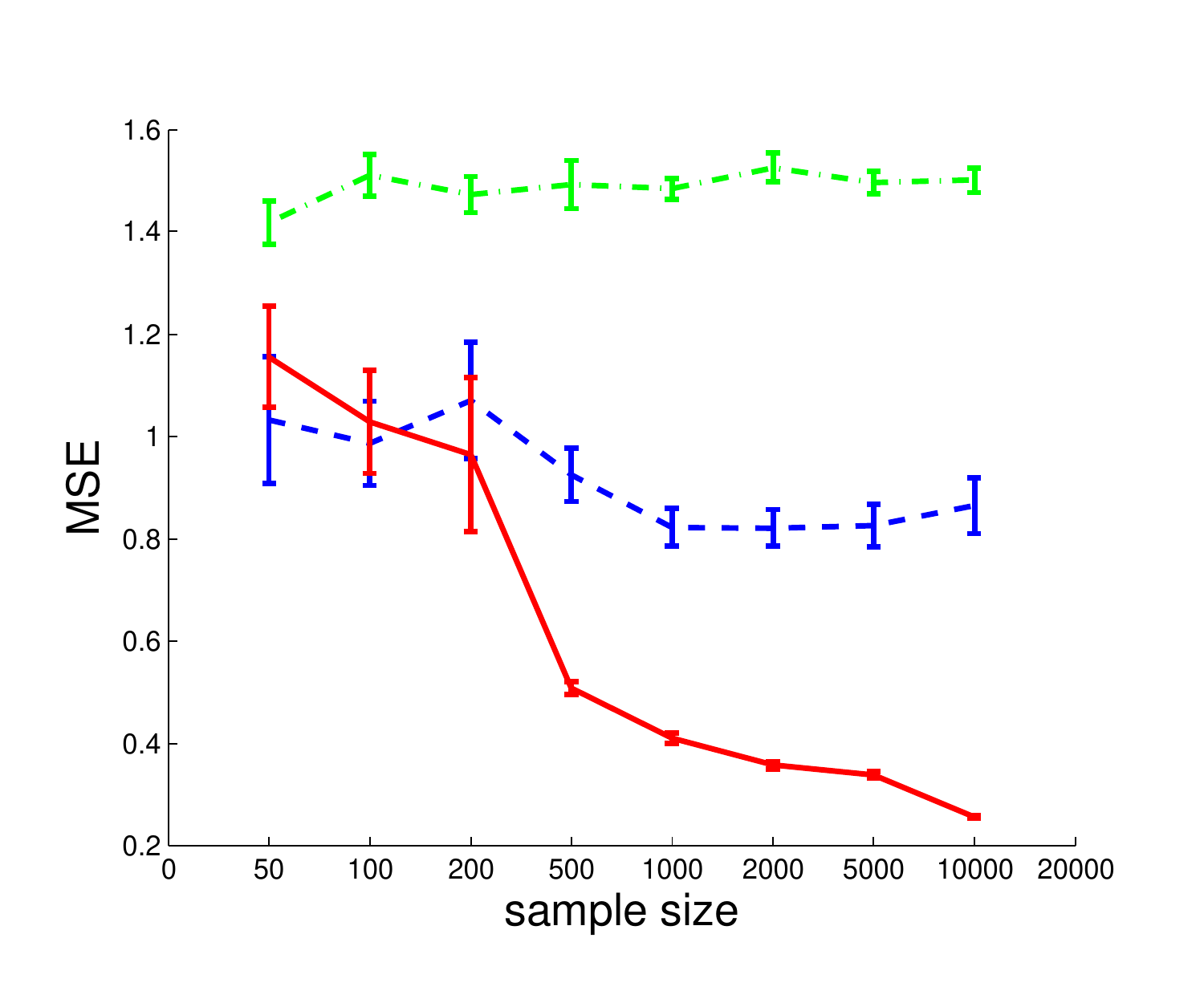} &
    \includegraphics[width=0.26\textwidth]{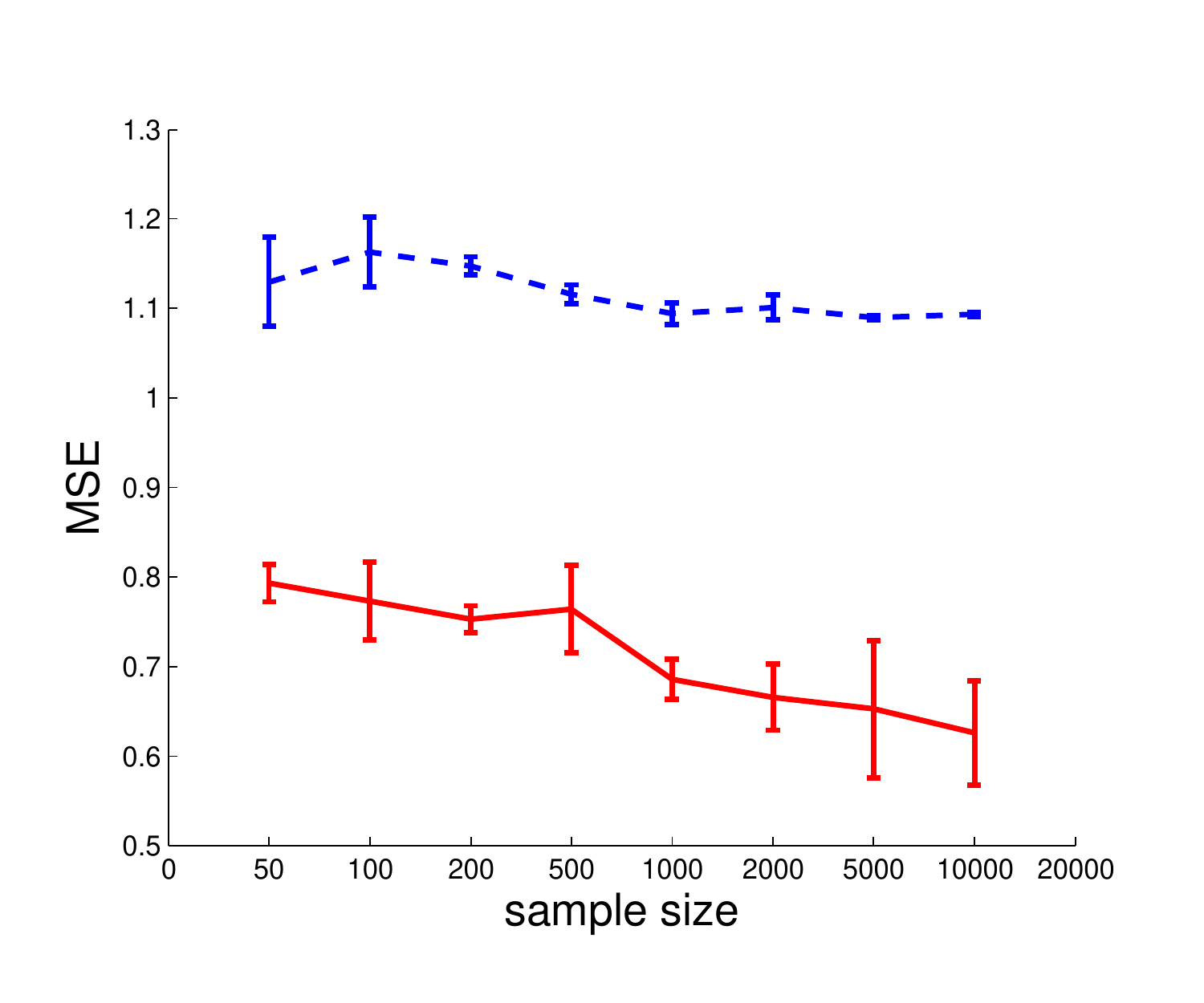} &
    \includegraphics[width=0.26\textwidth]{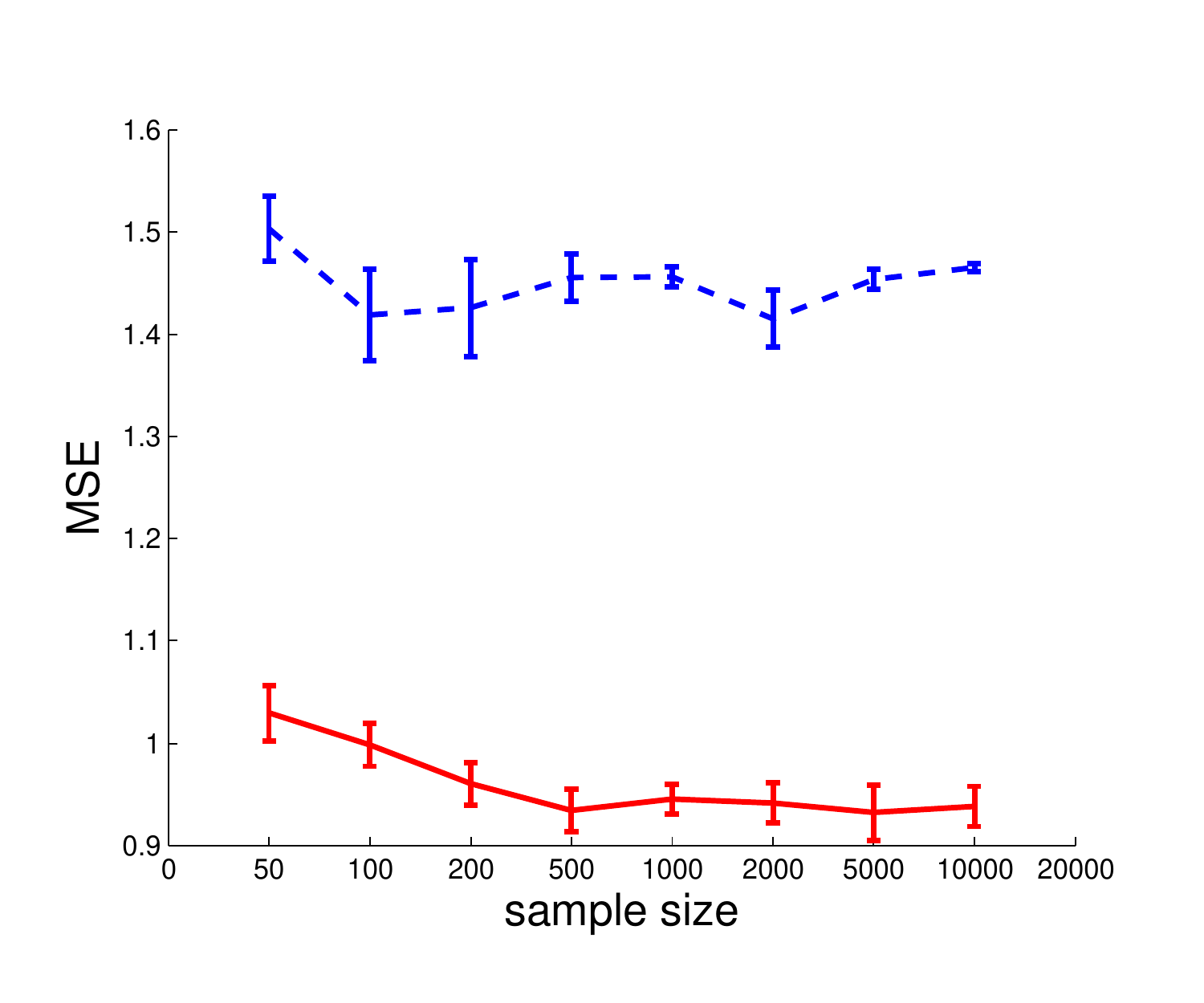} \\
    (e) Gaussian/Gamma $k=2$ & (f) Gaussian/Gamma $k=3$ & (g) Gaussian/Gamma $k=4$ & (h) Gaussian/Gamma $k=8$ \\
  \end{tabular}
  \vspace{-3mm}
  \caption{(a)-(d) Mixture of Gaussian distributions with $k=2,3,4,8$ components. (e)-(h) Mixture of Gaussian/Gamma distribution with $k=2,3,4,8$. For the former case, the performance of kernel spectral algorithm converge to those of EM algorithm for mixture of Gaussian model. For the latter case, the performance of kernel spectral algorithm are consistently much better than EM algorithm for mixture of Gaussian model. Spherical Gaussian spectral algorithm does not work for $k=4,8$, and hence not plotted.}\label{fig:synthetic}
  \vspace{-3mm}
\end{figure*}

We also varied the number of samples $m$ from $50$ to $10,000$, and
experimented with $k=2,3,4$ or $8$ mixture components. The mixture proportion for the $h$-th component is set to be $\pi_h= \frac{2h}{k(k+1)},~\forall h\in[k]$ (unbalanced). It is worth noting that as $k$ becomes larger, it is more difficult to recover parameters. This is because only a small number of data will be generated for the first several clusters. For every $n,k$ in each setting, we randomly generated 10 sets of samples and reported the average results.

{\bf Error measure.} We measured the performance of algorithms by the
following weighted $\ell_2$ norm difference
\begin{align*}
  MSE:=\sum_{h=1}^{k} \pi_h\, \sqrt{\sum_{j=1}^{m'} (p(x^j|h) - \widehat{p}(x^j|h))^2 },
\end{align*}
where $\{x^j\}_{j\in[m]}$ is a set of uniformly-spaced test points.

{\bf Results.} The results are plotted in Figure~\ref{fig:synthetic}. It is clear that the kernel spectral method converges rapidly with the data increment in all experiment settings.

In the mixture of Gaussians setting, the EM algorithm is best since the model is correctly specified. The spectral learning algorithm for spherical Gaussians does not perform well since the assumption of the method is too restricted. The performance of our kernel method converges to that of the EM algorithm.

In the mixture of Gaussian and Gamma setting, our kernel spectral algorithm achieves superior results compared to other algorithms. These results demonstrate that our algorithm is able to automatically adapt to the shape of the density.

We also plotted the actual recovered conditional densities in Figure~\ref{fig:shape}. The kernel spectral algorithm recovers nicely both the Gaussian and Gamma components, while the EM algorithm fails to fit one component.

We also note that the performance of EM degrades as the number of components increases, and our method outperforms EM in higher dimensions. This is the key advantage of our method in that it has favorable performance in higher dimensions, which agrees with the theoretical result in Theorem~\ref{thm:samplebound} that the sample complexity depends only quadratically in the number of components, when other parameters are held fixed.

\begin{figure}[t]
  \centering
  \subfigure[ EM using Mixture of Gaussians Model]
  {  \includegraphics[width=0.475\columnwidth]{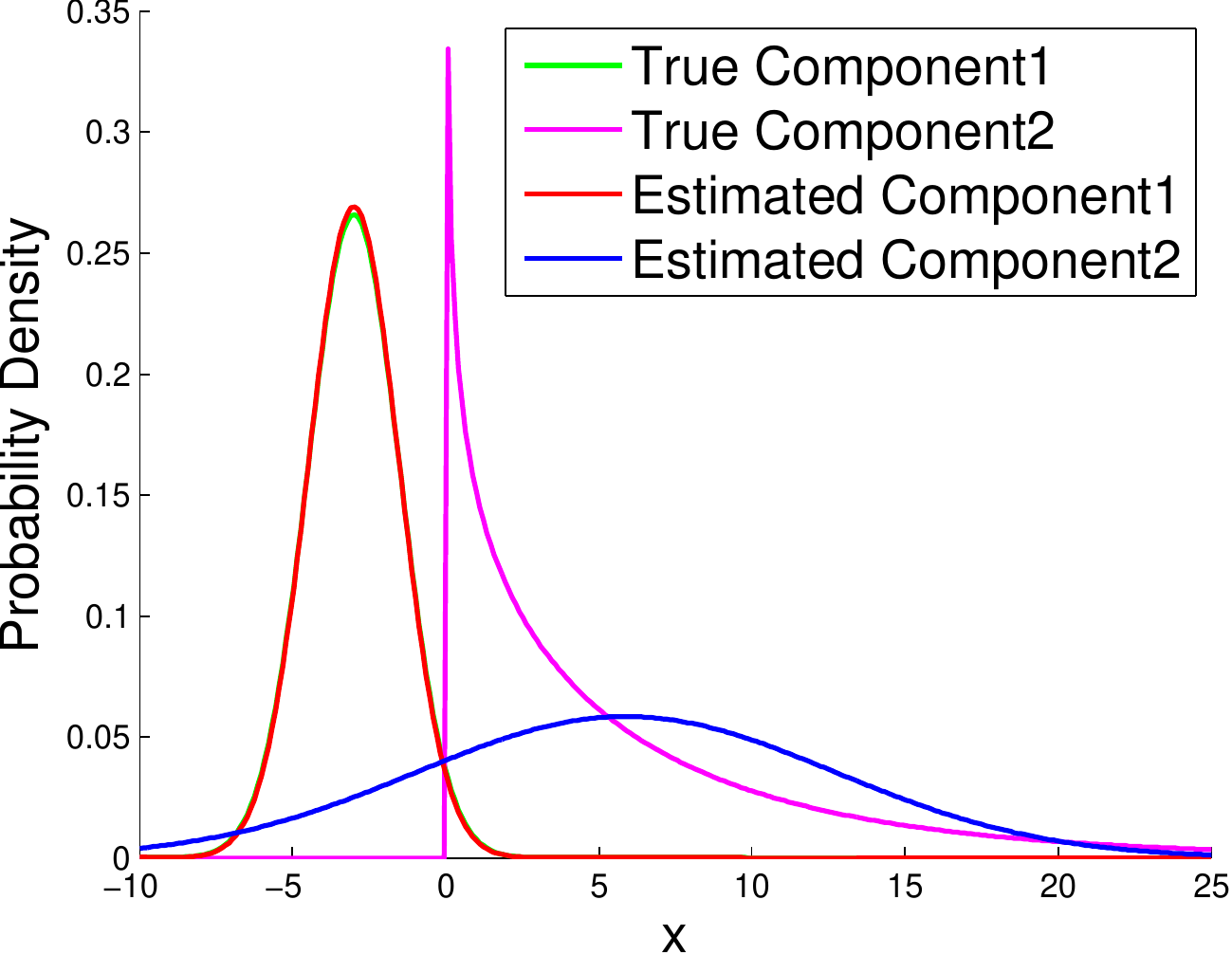}
  }
  \subfigure[Kernel Spectral]
  {
  \includegraphics[width=0.475\columnwidth]{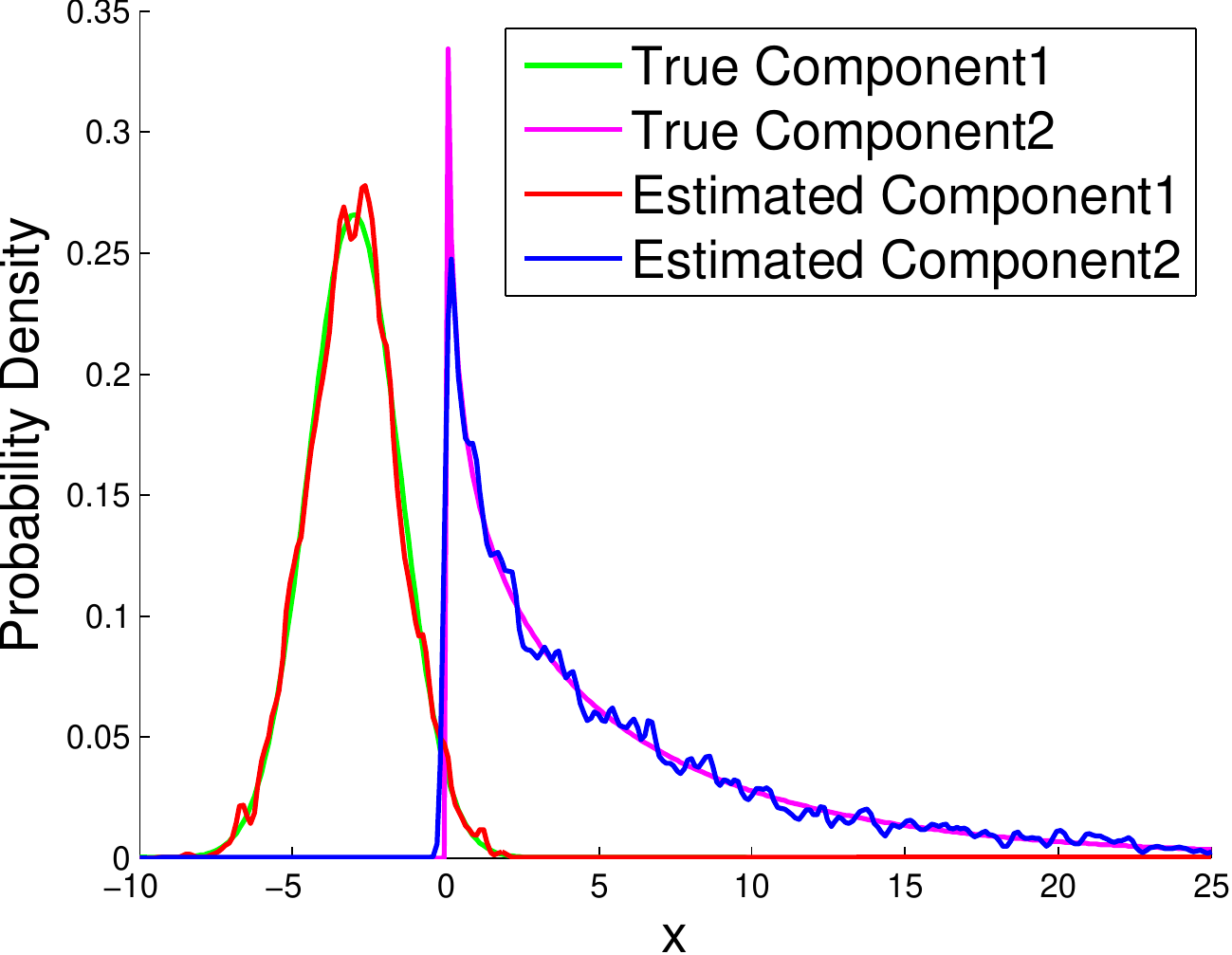}
  }
  \vspace{-3mm}
  \caption{Kernel spectral algorithm is able to adapt to the shape of the mixture components, while EM algorithm for mixture of Gaussians misfit the Gamma distribution.}\label{fig:shape}
\end{figure}

\subsubsection{Symmetric Case: Same Conditional Distribution}
\input{single_cond_exp.tex}

\subsection{Flow Cytometry Data}

Flow cytometry (FCM) data are multivariate measurements from flow cytometers that record light scatter and fluorescence emission properties of hundreds of thousands of individual cells. They are important to the studying of the cell structures of normal and abnormal cells and the diagnosis of human diseases. \citet{cytometry_nature} introduced the FlowCAP-challenge whose main task is grouping the flow cytometry data automatically. Clustering on the FCM data is a difficult task because the distribution of the data is non-Gaussian and heavily skewed.

We used the DLBCL Lymphoma dataset collection from~\cite{cytometry_nature} to compare our kernel algorithm with multi-view mixture of Gaussian model.
This collection contains 30 datasets, and each dataset consists of tens of thousands of cell measurements in 5 dimensions. Each dataset is a separate clustering task, and we fit a multi-view model to each dataset separately and use the maximum-a-posteriori assignment for obtaining the cluster labels.  All the cell measurements have been manually labeled, therefore we can evaluate the clustering performance using the f-score.

We split the 5 dimensional into three views: dimension 1 and 2 as the first view, 3 and 4 the second and 5 the third view. For each dataset, we select the best kernel bandwidth by 5-fold cross validation using log-likelihood. For EM algorithm for mixture of Gaussians (GMM) with diagonal covariances, we use a very generous 20 restarts. Figure~\ref{fig:real_data} presents the results sorted by the number of clusters. Our method (kernel spectral) outperforms EM-GMM in a majority of datasets. However, there are also datasets where kernel spectral algorithm has a large gap in performance compared to GMM. These are the datasets where the multi-view assumptions are heavily violated. Obtaining improved performance in these datasets will be a subject of our future study where we plan to develop even more robust kernel spectral algorithms.

\begin{figure*}[t!]
  \centering
  \begin{tabular}{c}
    \includegraphics[width=0.82\textwidth]{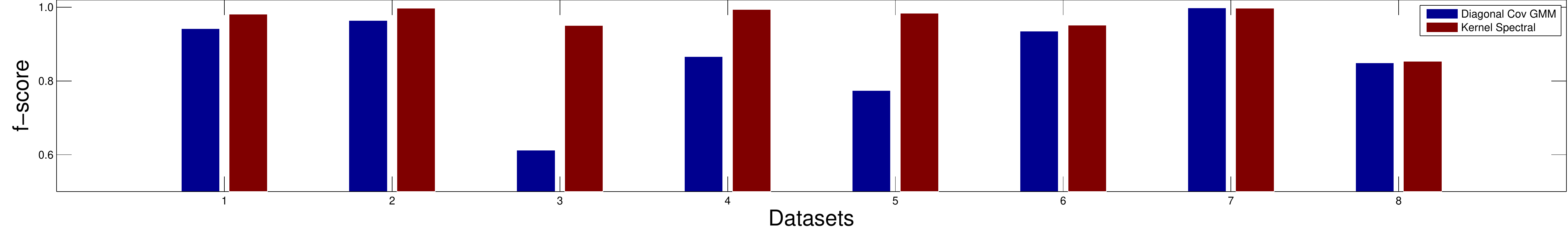} \\
    (a) number of clusters $k=2$ \\
    \includegraphics[width=0.98\textwidth]{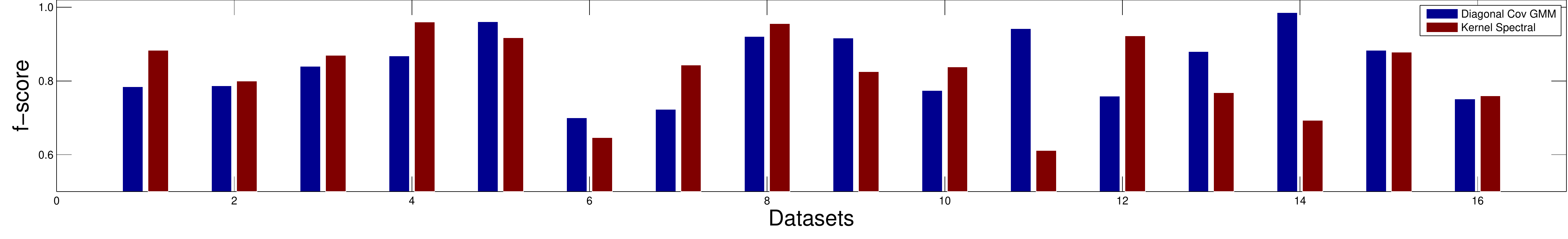}  \\
    (b) number of clusters $k=3$
  \end{tabular}
  \vspace{-3mm}
  \caption{Clustering results on two different datasets from the DLBCL flow cytometry data. Each group of bars represents F-scores from EM-GMM with diagonal covariances~(blue) and kernel spectral method~(red). The datasets are ordered by increasing sample size.}\label{fig:real_data}
\end{figure*}

\subsubsection*{Acknowledgements}
L. Song is supported in part by NSF Award IIS-1218749 and NIH 1RO1GM108341-01.
A. Anandkumar is supported
in part by Microsoft Faculty Fellowship, NSF Career award CCF-1254106, NSF Award CCF-1219234, and ARO YIP Award W911NF-13-1-0084.

\bibliographystyle{icml2014}
\bibliography{bibfile}



\begin{center}
{\Large Appendix}
\end{center}

\input{concentration-bounds.tex}

\end{document}

%% file: tensor-macros.tex
\def\tha{{\mbox{\tiny th}}}

\DeclareMathOperator\Diag{Diag}

\newcommand\R{\mathbb{R}}

\def\tl{\widetilde}
\def\h{\widehat}

\def\eps{\epsilon}

\DeclareMathOperator{\poly}{poly}

\renewcommand\th[1]{\ensuremath{\theta_{#1}}}

\newcommand\lambdamax{\ensuremath{\lambda_{\max}}}

\newcommand{\bp}{\begin{psfrags}}
\newcommand{\ep}{\end{psfrags}}
\newcommand{\bprfof}{\begin{proof_of}}
\newcommand{\eprfof}{\end{proof_of}}
\newcommand{\bprf}{\begin{myproof}}
\newcommand{\eprf}{\end{myproof}}

\newenvironment{myproof}{\noindent{\em Proof:} \hspace*{1em}}{
    \hspace*{\fill} $\Box$ }
\newenvironment{proof_of}[1]{\noindent {\em Proof of #1: }}{\hspace*{\fill} $\Box$ }

\def\beq{\begin{equation}}
\def\eeq{\end{equation}\noindent}
\def\beqn{\begin{eqnarray}}
\def\eeqn{\end{eqnarray} \noindent}
\def\beqnn{  \begin{eqnarray*}}
\def\eeqnn{\end{eqnarray*}  \noindent}
\def\bcase{  \begin{numcases}}
\def\ecase{\end{numcases}   \noindent}

%% file: single_cond_exp.tex


We also did some experiments for three-dimensional synthetic data that each view has the same conditional distribution. We generated the data from two settings:
\begin{itemize}
\item[1.] Mixture of Gaussian conditional density;
\item[2.] Mixture of Gaussian and shifted Gamma conditional density.
\end{itemize}
The mixture proportion and other experiment settings are exact same as the experiment in the main text. The only difference is that the conditional densities for each view here are the identical. We use the same measure to evaluate the performance. The empirical results are plotted in Figure~\ref{fig:sym_case}.

\begin{figure*}[!t]
  \renewcommand{\tabcolsep}{1pt}
  \begin{tabular}{cccc}
    \includegraphics[width=0.26\textwidth]{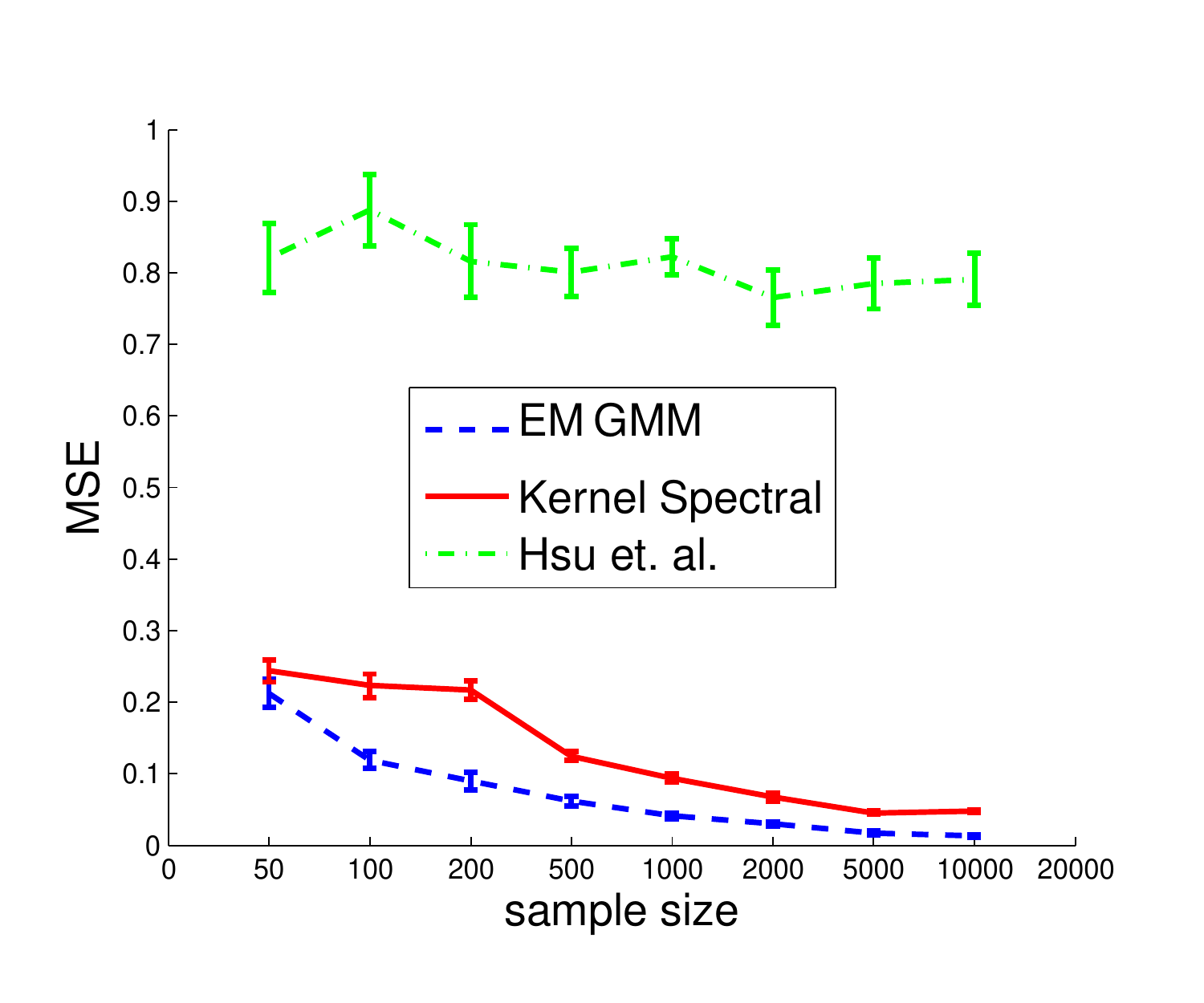} &
    \includegraphics[width=0.26\textwidth]{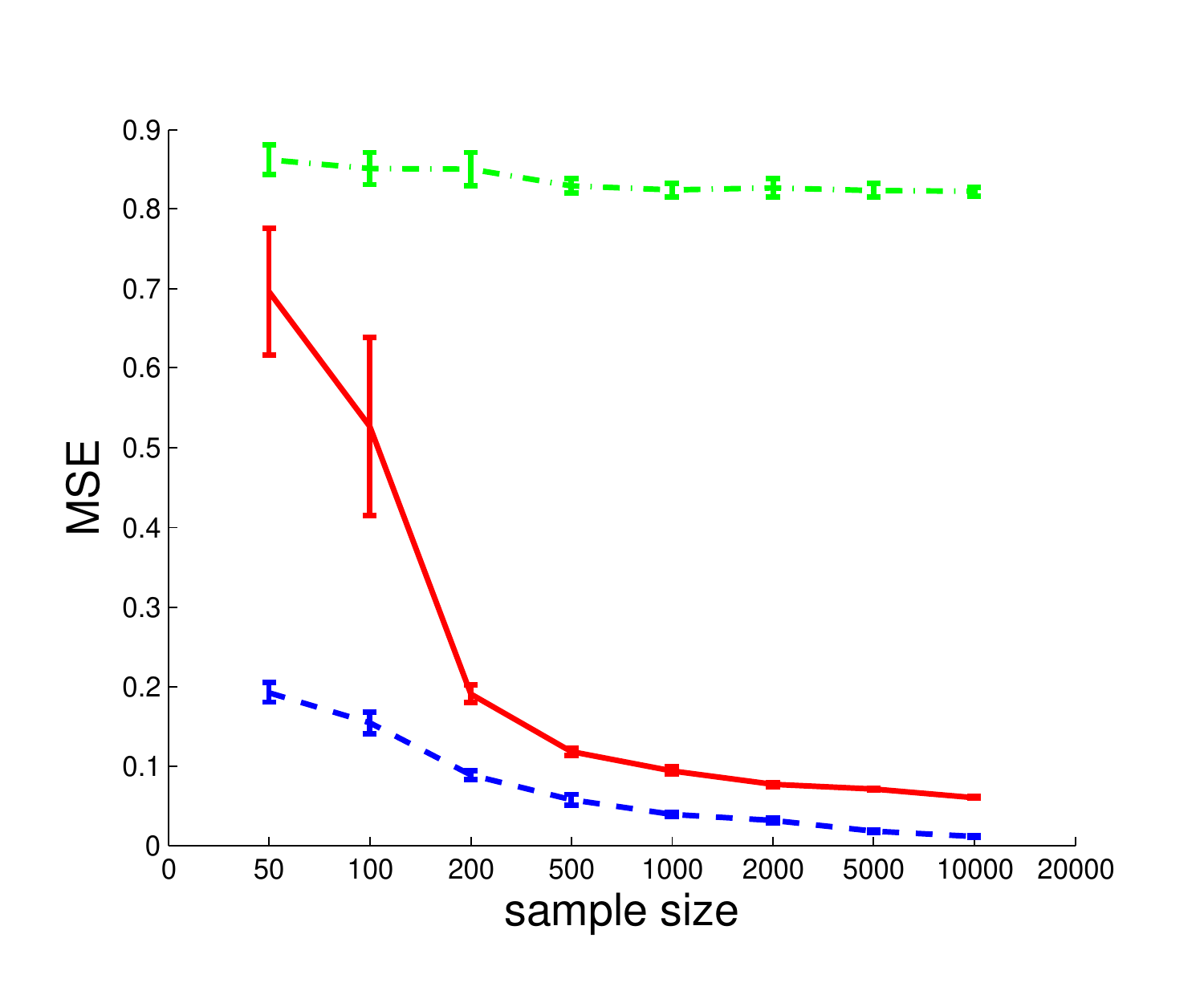} &
    \includegraphics[width=0.26\textwidth]{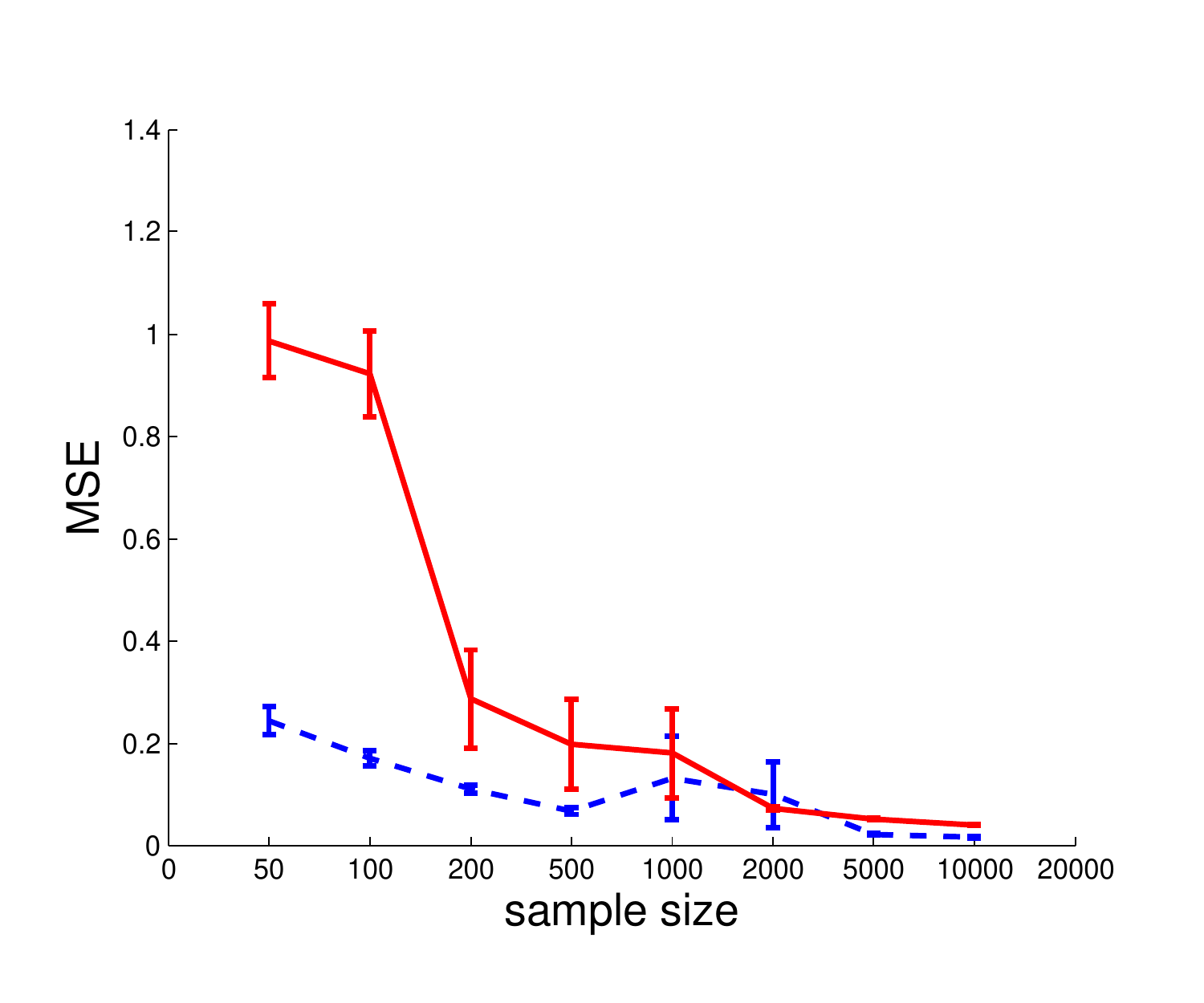} &
    \includegraphics[width=0.26\textwidth]{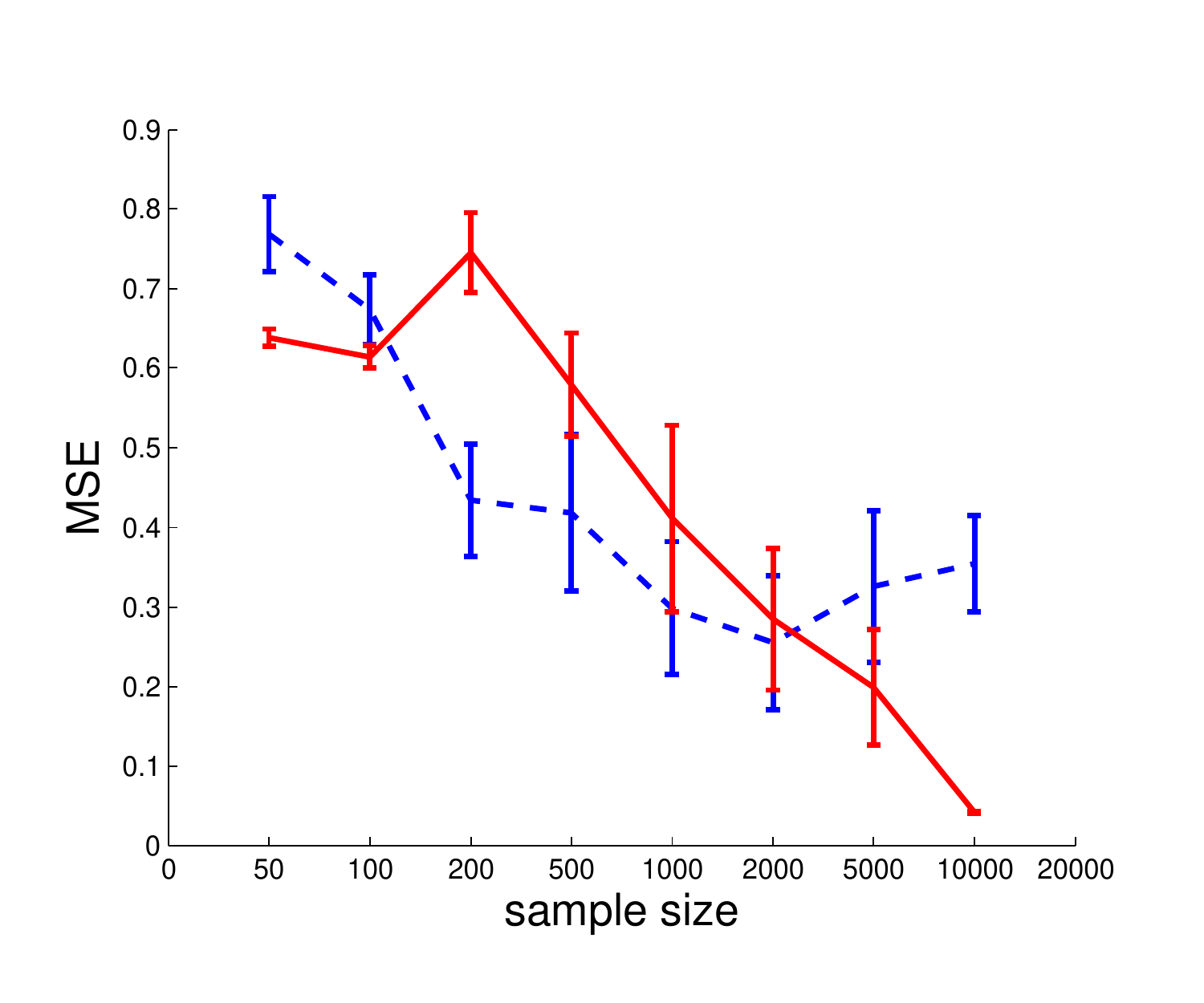} \\
    (a) Gaussian $k=2$ & (b) Gaussian $k=3$ & (c) Gaussian $k=4$ & (d) Gaussian $k=8$ \\
    \includegraphics[width=0.26\textwidth]{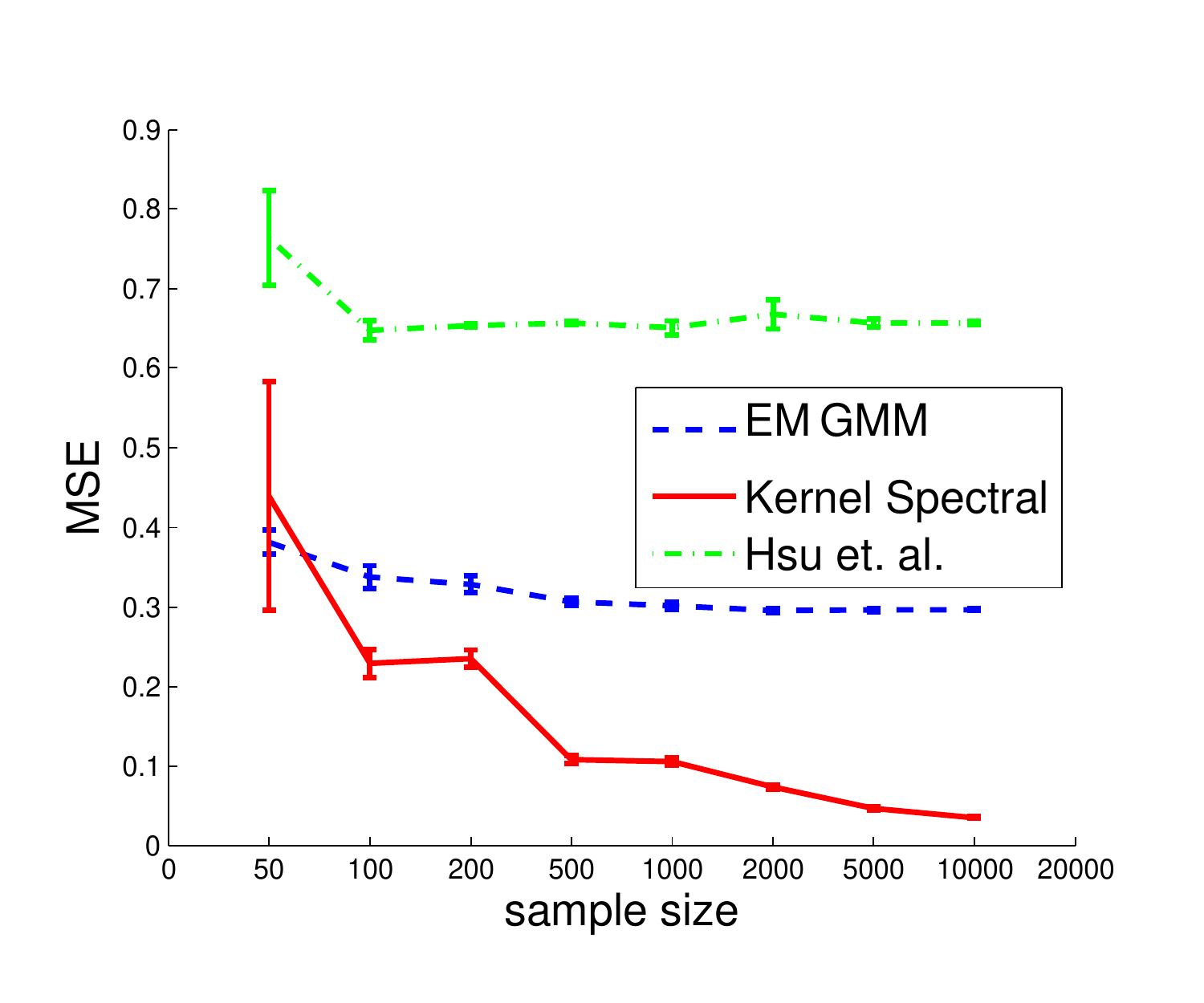} &
    \includegraphics[width=0.26\textwidth]{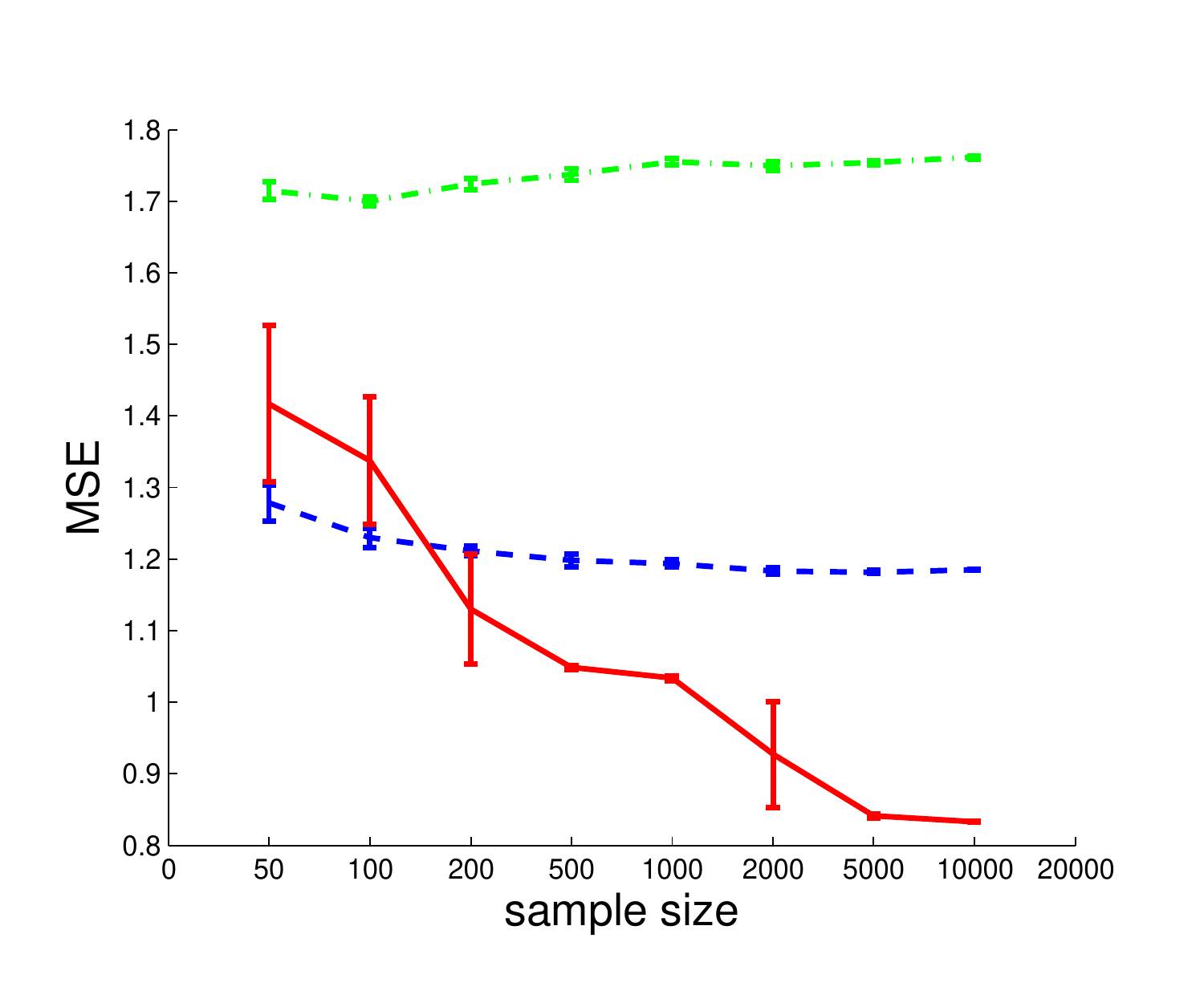} &
    \includegraphics[width=0.26\textwidth]{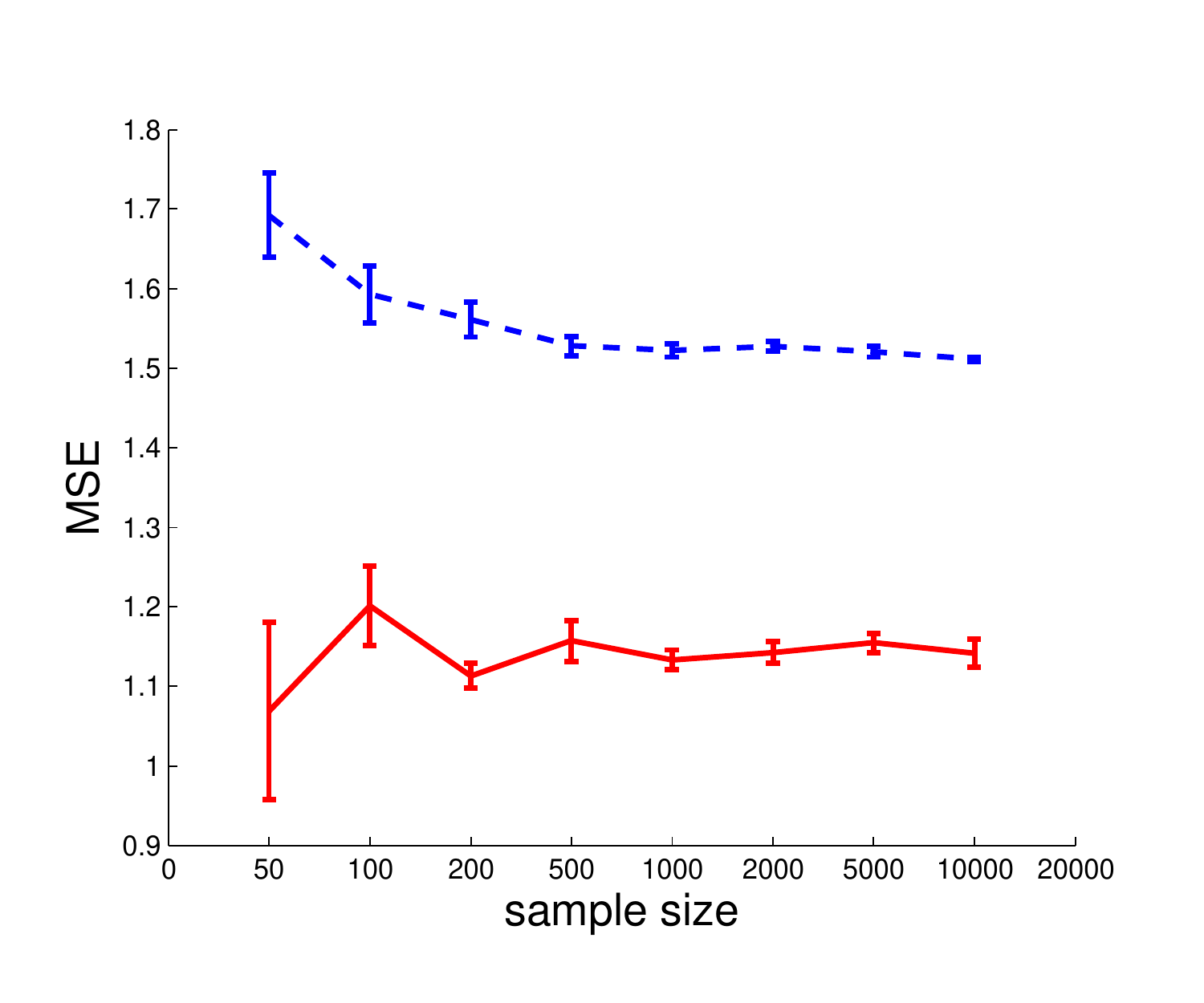} &
    \includegraphics[width=0.26\textwidth]{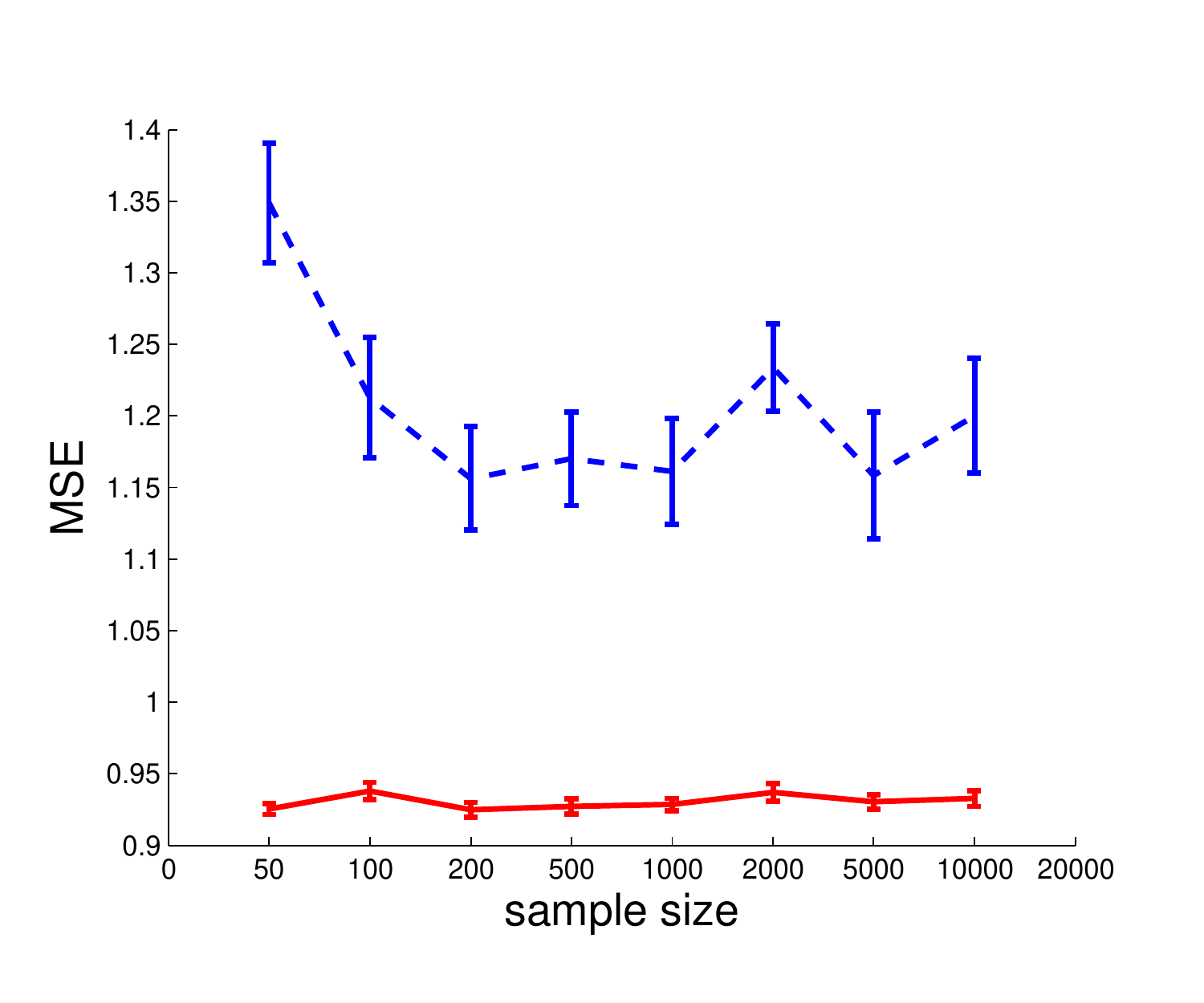} \\
    (e) Gaussian/Gamma $k=2$ & (f) Gaussian/Gamma $k=3$ & (g) Gaussian/Gamma $k=4$ & (h) Gaussian/Gamma $k=8$ \\
  \end{tabular}
  \caption{(a)-(d) Mixture of Gaussian distributions with $k=2,3,4,8$ components. (e)-(h) Mixture of Gaussian/Gamma distribution with $k=2,3,4,8$. For the former case, the performance of kernel spectral algorithm converge to those of EM algorithm for mixture of Gaussian model. For the latter case, the performance of kernel spectral algorithm are consistently much better than EM algorithm for mixture of Gaussian model. Spherical Gaussian spectral algorithm does not work for $k=4,8$, and hence not plotted.}\label{fig:sym_case}
\end{figure*}

As we expected, the behavior of the proposed method is similar to the results in different conditional densities case. In mixture of Gaussians, our algorithm converges to the EM GMM resuls. And in the mixture of Gaussian/shift Gamma, our algorithm consistently better to other alternatives.

%% file: concentration-bounds.tex
\section{Robust Tensor Power Method}
We recap the robust tensor power method for finding the tensor eigen-pairs in Algorithm~\ref{alg:robustpower}, analyzed in detail in~\cite{AnandkumarEtal:community12} and~\cite{AnandkumarEtal:tensor12}.
The method computes the eigenvectors of a   tensor through deflation, using a set of initialization vectors. Here, we employ random initialization vectors. This can be replaced with better initialization vectors, in certain settings, e.g. in the community model, the neighborhood vectors provide better initialization and lead to stronger guarantees~\cite{AnandkumarEtal:community12}.
Given the initialization vector, the method then runs a tensor power update, and runs for $N$ iterations to obtain an eigenvector. The successive eigenvectors are obtained via deflation.

\begin{algorithm}
\caption{$\{\lambda, \Phi\}\leftarrow $TensorEigen$(T,\, \{v_i\}_{i\in [L]}, N)$}\label{alg:robustpower}
\begin{algorithmic}
\renewcommand{\algorithmicrequire}{\textbf{Input: }}
\renewcommand{\algorithmicensure}{\textbf{Output: }}
\REQUIRE Tensor $T\in \R^{k \times k \times k}$, set of $L$ initialization vectors $\{v_i\}_{i\in L}$, number of
iterations  $N$.
\ENSURE the estimated eigenvalue/eigenvector pairs $\{\lambda, \Phi\}$, where $\lambda$ is the vector of eigenvalues and $\Phi$ is the matrix of eigenvectors.

\FOR{$i =1$ to $k$}
\FOR{$\tau = 1$ to $L$}
\STATE $\th{0}\leftarrow v_\tau$.
\FOR{$t = 1$ to $N$}
\STATE $\tilde{T}\leftarrow T$.
\FOR{$j=1$ to $i-1$ (when $i>1$)}
\IF{$|\lambda_j \inner{\th{t}^{(\tau)}, \phi_j}|>\xi$}
\STATE $\tilde{T}\leftarrow \tilde{T}- \lambda_j \phi_j^{\otimes 3}$.
\ENDIF
\ENDFOR

\STATE Compute power iteration update
$
\th{t}^{(\tau)}  :=
\frac{\tilde{T}(I, \th{t-1}^{(\tau)}, \th{t-1}^{(\tau)})}
{\|\tilde{T}(I, \th{t-1}^{(\tau)}, \th{t-1}^{(\tau)})\|}
$\ENDFOR
\ENDFOR

\STATE Let $\tau^* := \arg\max_{\tau \in L} \{ \tilde{T}(\th{N}^{(\tau)},
\th{N}^{(\tau)}, \th{N}^{(\tau)}) \}$.

\STATE Do $N$ power iteration updates starting from
$\th{N}^{(\tau^*)}$ to obtain eigenvector estimate $\phi_i$, and set $\lambda_i :=
\tilde{T}(\phi_i, \phi_i, \phi_i)$.

\ENDFOR
\RETURN the estimated eigenvalue/eigenvectors
$(\lambda, \Phi)$.

\end{algorithmic}
\end{algorithm}

\section{Proof of Theorem~\ref{thm:samplebound}}\label{app:samplebound}

\subsection{Recap of Perturbation Bounds for the Tensor Power Method}

We now recap the result of~\citet[Thm. 13]{AnandkumarEtal:community12} that establishes bounds on the eigen-estimates under good initialization vectors for the above procedure.
Let $\Tcal=\sum_{i\in [k]}\lambda_i v_i$, where $v_i$ are orthonormal vectors and $\lambda_1\geq \lambda_2\geq\ldots \lambda_k$. Let $\h{\Tcal}=\Tcal+E$ be the perturbed tensor with $\|E\|\leq \epsilon_{T}$. Recall that $N$ denotes the number of iterations of the tensor power method.
We call an initialization vector $u$ to be $(\gamma, R_0)$-good  if there exists $v_i$ such that $\inner{u}{v_i}> R_0$
  and $|\inner{u}{v_i}| -\max_{j<i} |\inner{u}{v_j}| > \gamma  |\inner{u}{v_i}|$.   Choose $\gamma=1/100$.

\begin{theorem}
\label{thm:robustpower}
There exists universal constants $C_1, C_2 > 0$  such that the
following holds.
\beq\label{eqn:robustpowerconditions}
\epsilon_{T} \leq C_1 \cdot \lambda_{\min} R_0^2,
\qquad
N \geq C_2 \cdot \left( \log(k) + \log\log\left(
\frac{\lambdamax}{\epsilon_T} \right) \right)
,
\eeq Assume there is at least one good initialization vector corresponding to each $v_i$, $i\in [k]$. The parameter $\xi$ for choosing deflation vectors in each iteration of the tensor power method in Procedure~\ref{alg:robustpower}  is chosen as $\xi\geq 25 \eps_T$. We obtain  eigenvalue-eigenvector pairs  $(\hat\lambda_1,\hat{v}_1), (\hat\lambda_2,\hat{v}_2), \dotsc,
(\hat\lambda_k,\hat{v}_k)$ such that  there exists a permutation $\eta$ on
$[k]$ with
\[
\|v_{\eta(j)}-\hat{v}_j\| \leq 8 \epsilon_T/\lambda_{\eta(j)}
, \qquad
|\lambda_{\eta(j)}-\hat\lambda_j| \leq 5\epsilon_T , \quad \forall j \in [k]
,
\]
and
\[
\left\|
\Tcal - \sum_{j=1}^k \hat\lambda_j \hat{v}_j^{\otimes 3}
\right\| \leq 55\eps_T .
\]
\end{theorem}

In the sequel, we establish concentration bounds that allows us to translate the above condition on tensor perturbation~\eqref{eqn:robustpowerconditions}  to sample complexity bounds.

\subsection{Concentration Bounds}

\subsubsection{Analysis of Whitening}

Recall that we use the covariance operator $\Ccal_{X_1 X_2}$ for whitening the 3rd order embedding $\Ccal_{X_1, X_2, X_3}$. We first analyze the perturbation in whitening when sample estimates are employed.

Let $\h{\Ccal}_{X_1 X_2}$ denote the sample covariance operator between variables $X_1$ and $X_2$, and let \[B:=0.5(\h{\Ccal}_{X_1 X_2}+ \h{\Ccal}_{X_1 X_2}^\top)=\h{\Ucal}\h{S}\h{\Ucal}^\top\] denote the SVD.
Let $\h{\Ucal}_k$ and $\h{S}_k$ denote the restriction to top-$k$ eigen-pairs, and let $B_{k} := \h{\Ucal}_k \h{S}_k \h{\Ucal}_k^\top$. Recall that the whitening matrix is given by $\h{\Wcal}:=\h{\Ucal}_k \h{S}_k^{-1/2}$. Now $\h{\Wcal}$ whitens $B_k$, i.e. $\h{\Wcal}^\top B_{k} \h{\Wcal}=I$.

Now consider the SVD of
\[ \h{\Wcal}^\top \Ccal_{X_1 X_2} \h{\Wcal}= A D A^\top,\] and define \[\Wcal:= \h{\Wcal} AD^{-1/2}A^\top, \] and $\Wcal$ whitens $\Ccal_{X_1 X_2}$ since $\Wcal^\top  \Ccal_{X_1 X_2} W=I$.
Recall that by exchangeability assumption,
\beq\label{eqn:pairsexpression} \Ccal_{X_1,X_{2}}
  = \sum_{j=1}^k \pi_j \cdot \mu_{X|j} \otimes \mu_{X|j} = M \Diag(\pi) M^\top , \eeq where the $j^{\tha}$ column of $M$, $M_j = \mu_{X|j}$.

We now establish the following perturbation bound on the whitening procedure. Recall from \eqref{eqn:deltapairs}, $ \epsilon_{pairs}:=\nbr{\Ccal_{X_1,X_{2}} - \widehat \Ccal_{X_1,X_{2}}}_{}$. Let $\sigma_1(\cdot) \geq \sigma_2(\cdot)\ldots$ denote the singular values of an operator.

\begin{lemma}[Whitening perturbation]\label{lemma:whiten} Assuming that $\epsilon_{pairs} < 0.5 \sigma_k(\Ccal_{X_1 X_2})$,
\beq \epsilon_{W}:= \|\Diag(\pi)^{1/2}M^\top(\h{\Wcal}-\Wcal)\|\leq \frac{4\epsilon_{pairs} \iffalse+2\sigma_{k+1}(\Ccal_{X_1 X_2})\fi}{ \sigma_{k}(\Ccal_{X_1 X_2})}\eeq
\end{lemma}

\paragraph{Remark: }Note that $\sigma_{k}(\Ccal_{X_1 X_2}) = \sigma_{k}^2(M)$.

\bprf The proof is along the lines of Lemma 16 of~\cite{AnandkumarEtal:community12}, but adapted to whitening using the covariance operator here.
 \begin{align*}\|\Diag(\pi)^{1/2} M^\top(\h{\Wcal}-\Wcal)\|&=
\|\Diag(\pi)^{1/2} M^\top W(A D^{1/2} A^\top -I)\|\\ &\leq\|\Diag(\pi)^{1/2} M^\top \Wcal\| \|D^{1/2}-I\|. \end{align*} Since $\Wcal$ whitens $\Ccal_{X_1 X_2}=M \Diag(\pi) M^\top$, we have that $\|\Diag(\pi)^{1/2} M^\top \Wcal\| =1$. Now we control $\|D^{1/2}-I\|$.  Let $\tl{E}:= \Ccal_{X_1,X_{2}} -B_k$, where recall that $B=0.5( \widehat \Ccal_{X_1,X_{2}}+ \h{\Ccal}_{X_1 X_2}^\top)$ and $B_k$ is its restriction to top-$k$ singular values. Thus, we have $\|\tl{E}\| \leq \epsilon_{pairs} + \sigma_{k+1}(B)\leq 2\epsilon_{pairs}$.
 We now have
\begin{align*}
\|D^{1/2}-I\|&\leq \|(D^{1/2}-I)(D^{1/2}+I)\|\leq \|D-I\|
\\ &=\|AD A^\top - I\| = \|\h{\Wcal}^\top \Ccal_{X_1 X_2}  \h{\Wcal} -I\|\\ &=\| \h{\Wcal}^\top  \tl{E} \h{\Wcal}\| \leq \|\h{\Wcal}\|^2 ( 2 \epsilon_{pairs}).
\end{align*}Now
\[ \|\h{\Wcal}^2\| \leq\frac{1}{ \sigma_k(\h{\Ccal}_{X_1 X_2})}\leq \frac{2}{\sigma_k(\Ccal_{X_1 X_2})},\] when  $\epsilon_{pairs}<0.5 \sigma_k(\Ccal_{X_1 X_2})$.
\eprf

\subsubsection{Tensor Concentration Bounds}

Recall that the whitened tensor from samples is given by
$$\h{\Tcal} := \h{\Ccal}_{X_1 X_2 X_3} \times_1 (\h{\Wcal}^\top) \times_2 (\h{\Wcal}^\top) \times_3 (\h{\Wcal}^\top).$$ We want to establish its perturbation from the whitened tensor using exact statistics
$$\Tcal := \Ccal_{X_1 X_2 X_3} \times_1 (\Wcal^\top ) \times_2 (\Wcal^\top ) \times_3 (\Wcal^\top ).$$ Further, we have
\beq\label{eqn:triplesexpression}\Ccal_{X_1 X_2 X_3}= \sum_{h\in [k]} \pi_h \cdot \mu_{X|h} \otimes \mu_{X|h} \otimes \mu_{X|h} \eeq

Let $\epsilon_{triples}:= \|\h{\Ccal}_{X_1 X_2 X_3}-\Ccal_{X_1 X_2 X_3}\|_{}$. Let $\pi_{\min}:=\min_{h\in [k]}\pi_h$.

\begin{lemma}[Tensor perturbation bound]
Assuming that $\epsilon_{pairs} < 0.5 \sigma_k(\Ccal_{X_1 X_2})$, we have
\beq\label{eqn:epsilonT} \epsilon_T:= \|\h{\Tcal} - \Tcal\|
\leq \frac{2\sqrt{2}\epsilon_{triples}}{\sigma_k(\Ccal_{X_1 X_2})^{1.5} }+\frac{\epsilon_W^3}{\sqrt{\pi_{\min}}}.\eeq
\end{lemma}

\bprf  Define
  intermediate tensor
\begin{align*} \tl{\Tcal}&:= \Ccal_{X_1 X_2 X_3} \times_1 (\h{\Wcal}^\top) \times_2 (\h{\Wcal}^\top) \times_3 (\h{\Wcal}^\top).\end{align*}
We will bound $\|\h{\Tcal}-\tl{\Tcal}\|$  and $\| \h{\Tcal}-\Tcal\|$  separately.
\begin{align*}
\|\h{\Tcal}-\tl{\Tcal}\| &\leq \|\h{\Ccal}_{X_1, X_2, X_2} - \Ccal_{X_1, X_2, X_3}\| \|\h{\Wcal}\|^3\leq \frac{2\sqrt{2}\epsilon_{triples}}{\sigma_k(\Ccal_{X_1 X_2})^{1.5} },
\end{align*}using the bound on $\|\h{\Wcal}\|$ in Lemma~\ref{lemma:whiten}. For the other term,
first note that
\[ \Ccal_{X_1, X_2, X_3} = \sum_{h\in [k]} \pi_h \cdot M_h \otimes M_h \otimes M_h , \]
\begin{align*} \|\h{\Tcal}-\Tcal\|&= \| \Ccal_{X_1 X_2 X_3}\times_1 (\h{\Wcal} -\Wcal)^\top \times_2 (\h{\Wcal} -\Wcal)^\top \times_3 (\h{\Wcal}-\Wcal)^\top\| \\
&\leq \frac{ \| \Diag(\pi)^{1/2}M^\top(\h{\Wcal}-\Wcal)\|^3}{\sqrt{\pi_{\min}}}
\\
&= \frac{\epsilon_W^3}{\sqrt{\pi_{\min}}}
\end{align*}
\eprf\\

\bprfof{Theorem~\ref{thm:samplebound}}
We obtain a condition on the above perturbation $\epsilon_T$ in \eqref{eqn:epsilonT} by applying Theorem~\ref{thm:robustpower} as
$ \epsilon_T\leq C_1\lambda_{\min} R_0^2$. Here, we have $\lambda_{i} = 1/\sqrt{\pi_{i}}\geq 1$. For random initialization, we have that $R_0 \sim 1/\sqrt{k}$, with probability $1-\delta$ using $\poly(k) \poly(1/\delta)$ trials, see Thm. 5.1 in~\cite{AnandkumarEtal:tensor12}. Thus, we require that $ \epsilon_T  \leq \frac{C_1}{k}$. Summarizing, we require for the following conditions to hold
\beq\epsilon_{pairs}\leq 0.5 \sigma_k(\Ccal_{X_1 X_2}), \quad \epsilon_T  \leq \frac{C_1}{k}.\eeq
We now substitute for $\epsilon_{pairs}$ and $\epsilon_{triples}$ in \eqref{eqn:epsilonT} using Lemma~\ref{lemma:pairs} and Lemma~\ref{lemma:triples}.



From Lemma~\ref{lemma:pairs}, we have that
\[ \epsilon_{pairs}  \leqslant \frac{2\sqrt{2}\rho \sqrt{\log\frac{\delta}{2}}}{\sqrt{m}}, \]with probability $1-\delta$. It is required that $\epsilon_{pairs} < 0.5 \sigma_k(\Ccal_{X_1, X_2})$, which yields that \beq\label{eqn:cond1} m > \frac{32 \rho^2 \log\frac{\delta}{2}}{\sigma^2_k(\Ccal_{X_1, X_2})}.\eeq
Further we require that $\epsilon_T \leq C_1/k$, which implies that each of the terms in \eqref{eqn:epsilonT} is less than $C/k$, for some constant $C$. Thus, we have
\[ \frac{2\sqrt{2} \epsilon_{triples}}{\sigma_k^{1.5}(\Ccal_{X_1, X_2})} < \frac{C}{k}\quad\Rightarrow\quad m > \frac{C_3 k^2 \rho^3 \log \frac{\delta}{2}}{\sigma_k^{3}(\Ccal_{X_1, X_2})},\]
for some constant $C_3$ with probability $1-\delta$ from Lemma~\ref{lemma:triples}. Similarly for the second term in \eqref{eqn:epsilonT}, we have
\[\frac{\epsilon_W^3}{\sqrt{\pi_{\min}}}< \frac{C}{k},\]and from Lemma~\ref{lemma:whiten}, this implies that \[ \epsilon_{pairs} \leq \frac{C' \pi_{\min}^{1/6} \sigma_k(\Ccal_{X_1, X_2})}{k^{1/3}\iffalse(1+\sigma_{k+1}(\Ccal_{X_1, X_2}))\fi},\]Thus, we have
\[  m > \frac{C_4 k^{\frac{2}{3}} \rho^2 \log\frac{\delta}{2}\iffalse (1+\sigma_{k+1}(\Ccal_{X_1, X_2}))^2\fi}{\pi_{\min}^{\frac{1}{3}} \sigma^2_k(\Ccal_{X_1, X_2})}, \]for some other constant $C_4$ with probability $1-\delta$. 
 Thus, we have the result in Theorem~\ref{thm:samplebound}.

\eprfof

\subsubsection{Concentration bounds for Empirical Operators}

Concentration results for the singular value decomposition of empirical operators.

\begin{lemma}[Concentration bounds for pairs]\label{lemma:pairs} Let $\rho:=\sup_{x \in \Omega} k(x,x)$, and $\| \cdot\|_{}$ be the Hilbert-Schmidt norm, we have for \beq \epsilon_{pairs}:=\nbr{\Ccal_{X_1 X_2} - \widehat \Ccal_{X_1 X_2}}_{},\label{eqn:deltapairs} \eeq
\begin{eqnarray}
	\Pr \cbr{\epsilon_{pairs}  \leqslant \frac{2\sqrt{2}\rho \sqrt{\log\frac{\delta}{2}}}{\sqrt{m}} } \geqslant 1-\delta. \label{eq:operator_concentration}
\end{eqnarray}
\end{lemma}

\begin{proof}
We will use similar arguments as in~\cite{RosBelVit2010} which deals with symmetric operator. Let $\xi_{i}$ be defined as
\begin{eqnarray}
\xi_{i}\, =\, \phi(x_1^i) \otimes \phi(x_2^i) - \Ccal_{X_1,X_2}.
\end{eqnarray}
It is easy to see that $\mathbb{E}[\xi_{i}] = 0$. Further, we have
\begin{eqnarray}
	\sup_{x_1,x_2} \nbr{\phi(x_1) \otimes \phi(x_2)}^{2}_{}
    = \sup_{x_1,x_2} k(x_1, x_1) k(x_2,x_2)
    \leqslant \rho^{2},
\end{eqnarray}
which implies that $\nbr{\Ccal_{X_1 X_2}}_{} \leqslant \rho$, and $\nbr{\xi_i}_{} \leqslant 2 \rho$. The result then follows from the Hoeffding's inequality in Hilbert space.
\end{proof}

Similarly, we have the concentration bound for 3rd order embedding.

\begin{lemma}[Concentration bounds for triples]\label{lemma:triples} Let $\rho:=\sup_{x \in \Omega} k(x,x)$, and $\| \cdot\|_{}$ be the Hilbert-Schmidt norm, we have for \beq \epsilon_{triples}:=\nbr{\Ccal_{X_1 X_2 X_3} - \widehat \Ccal_{X_1 X_2 X_3}}_{},\label{eqn:deltapairs} \eeq
\begin{eqnarray}
	\Pr \cbr{\epsilon_{triples}  \leqslant \frac{2\sqrt{2}\rho^{3/2} \sqrt{\log\frac{\delta}{2}}}{\sqrt{m}} } \geqslant 1-\delta. \label{eq:operator_concentration2}
\end{eqnarray}
\end{lemma}

\begin{proof}
We will use similar arguments as in~\cite{RosBelVit2010} which deals with symmetric operator. Let $\xi_{i}$ be defined as
\begin{eqnarray}
\xi_{i}\, =\, \phi(x_1^i) \otimes \phi(x_2^i) \otimes \phi(x_3^i)  - \Ccal_{X_1 X_2 X_3}.
\end{eqnarray}
It is easy to see that $\mathbb{E}[\xi_{i}] = 0$. Further, we have
\begin{eqnarray}
	\sup_{x_1,x_2,x_3} \nbr{\phi(x_1) \otimes \phi(x_2) \otimes \phi(x_3)}^{2}_{}
    = \sup_{x_1,x_2,x_3} k(x_1, x_1) k(x_2,x_2) k(x_3,x_3)
    \leqslant \rho^{3},
\end{eqnarray}
which implies that $\nbr{\Ccal_{X_1 X_2 X_3}}_{} \leqslant \rho^{3/2}$, and $\nbr{\xi_i}_{} \leqslant 2 \rho^{3/2}$. The result then follows from the Hoeffding's inequality in Hilbert space.
\end{proof}